%% file: Main.tex
\documentclass[lettersize,journal]{IEEEtran}

\usepackage[utf8]{inputenc}
\usepackage[T1]{fontenc}
\usepackage{textcomp}

\usepackage{amsmath,amssymb,amsfonts}
\usepackage{amsthm}
\usepackage{mathtools}
\usepackage{bm}

\usepackage{aliascnt}

\newtheorem{theorem}{Theorem}%[section]
\newtheorem{lemma}{Lemma}%[section]
\newtheorem{corollary}{Corollary}%[section]
\newtheorem{proposition}{Proposition}%[section]
\newtheorem{definition}{Definition}%[section]
\newtheorem{remark}{Remark}%[section]

\usepackage{algorithm}
\usepackage{algorithmic}

\usepackage{array}
\usepackage{booktabs}
\usepackage{enumitem}
\usepackage[caption=false,font=normalsize,labelfont=sf,textfont=sf]{subfig}
\usepackage{stfloats}

\usepackage{graphicx}
\usepackage{xcolor}
\usepackage{tikz}
\usetikzlibrary{arrows.meta, positioning, fit, calc,
                decorations.pathmorphing, decorations.markings,
                decorations.text, decorations.pathreplacing,
                shapes.geometric, backgrounds}

\usepackage{siunitx}

\usepackage{url}
\usepackage{verbatim}
\usepackage{cite}
\usepackage{hyperref}
\hypersetup{colorlinks=true, linkcolor=black, citecolor=black, urlcolor=blue, pdfborder={0 0 0}}
\usepackage[capitalise,noabbrev]{cleveref}
\crefname{theorem}{Theorem}{Theorems}
\Crefname{theorem}{Theorem}{Theorems}
\crefname{lemma}{Lemma}{Lemmas}
\Crefname{lemma}{Lemma}{Lemmas}
\crefname{corollary}{Corollary}{Corollaries}
\Crefname{corollary}{Corollary}{Corollaries}
\crefname{proposition}{Proposition}{Propositions}
\Crefname{proposition}{Proposition}{Propositions}
\crefname{definition}{Definition}{Definitions}
\Crefname{definition}{Definition}{Definitions}
\crefname{remark}{Remark}{Remarks}
\Crefname{remark}{Remark}{Remarks}

\newcommand{\SOthree}{\mathrm{SO}(3)}

\newcommand{\R}{\mathbb{R}}

\newcommand{\abs}[1]{\lvert #1 \rvert}

\newcommand{\vect}[1]{\boldsymbol{#1}}
\newcommand{\mat}[1]{\boldsymbol{#1}}

\begin{document}

\title{Passive Fault Tolerance through Tension-to-Thrust Feed-Forward: Hybrid Input-to-State Stability for Decentralized Multi-UAV Slung-Load Transport under Abrupt Cable Severance}
\author{Hadi Hajieghrary$^{1}$ and Paul Schmitt$^{2}$%
\thanks{This work received no external research funding.}%
\thanks{$^{1}$Hadi Hajieghrary presents this paper as an Independent Researcher (e-mail: Hadi.Hajieghrary@gmail.com).}%
\thanks{$^{2}$Paul Schmitt is with Reynolds \& Moore (e-mail: Paul.Schmitt@reynolds-moore.com).}}

% The paper headers
\markboth{IEEE TRANSACTIONS ON CONTROL SYSTEMS TECHNOLOGY}%
{Hajieghrary and Schmitt: Passive Fault Tolerance through Tension-to-Thrust Feed-Forward}

\maketitle

\begin{abstract}
Abrupt cable severance in multi-UAV slung-load transport redistributes load and changes the active constraint set, leaving limited time for fault diagnosis and reconfiguration. Existing controllers rely on coordinated force allocation, peer-state exchange, or fixed cable topology, and therefore lack a certified decentralized recovery mechanism for unannounced severance. We present a passive architecture that routes each vehicle's measured cable tension directly into its altitude thrust command, $T_i^{\mathrm{ff}}=T_i$, while a surrounding proportional--derivative, anti-swing, and projection cascade preserves local tracking feasibility. The main contribution is a conditional hybrid practical input-to-state-stability certificate that composes a slack-excursion-bounded taut-cable reduction, bounded post-severance Lyapunov jumps, inter-fault decay, and per-fault-cycle contraction $\rho \in (0,1)$ into an explicit recovery envelope under stated actuator, slack, and dwell assumptions. We validate the controller in Drake multibody simulation with five vehicles, a 10\,kg payload, Kelvin--Voigt cables, Dryden wind, and single- and dual-severance schedules: the closed loop attains 0.312--0.328\,m RMSE, 76.1--95.2\,mm peak sag, and recovery within one payload-pendulum period. Disabling the identity inflates cruise error by 34--39\% and peak sag by 3.6--4.0$\times$, identifying local tension feed-forward as the dominant passive recovery mechanism in the tested decentralized cascade.
\end{abstract}

\begin{IEEEkeywords}
UAV Planning and Control, Fault tolerant systems, Decentralized control, Multibody dynamics
\end{IEEEkeywords}

\section{Introduction}
\label{sec:Introduction}
\input{Section_I_Introduction.tex}

\section{Problem Statement}
\label{sec:ProblemStatement}
\label{sec:Problem_Statement}
\input{Section_II_Problem_Statement.tex}

\section{Proposed Method}
\label{sec:ProposedMethod}
\label{sec:proposed}
\input{Section_III_Proposed_Method.tex}

\section{Stability Analysis}
\label{sec:StabilityAnalysis}
\label{sec:stability}
\input{Section_IV_Stability_Analysis.tex}

\section{Extension Layers and Theorem Preservation}
\label{sec:Parameter_Adaptation_Tension_Constraints_and_Post-Fault_Formation_Reshape}
\label{sec:extensions}
\input{Section_V_Parameter_Adaptation_Tension_Constraints_and_Post-Fault_Formation_Reshape.tex}

\section{Simulation Campaign and Results}
\label{sec:simulation}
\input{Section_VI_Simulation_Campaign_and_Results.tex}

\section{Discussion}
\input{Section_VII_Discussion.tex}

\section{Conclusion}
\input{Section_VIII_Conclusion.tex}

% ======================== REFERENCES ============================
\bibliographystyle{IEEEtran}
\bibliography{References}

\end{document}

%% file: Section_I_Introduction.tex
Cooperative aerial transport uses multiple uncrewed aerial vehicles (UAVs) to carry a shared payload through flexible cables. In practical deployments, cable severance can occur due to abrasion, hardware fatigue, or contact with obstacles. A severance event changes the system's constraint structure, redistributes load among the remaining cables, and requires a compensatory response within one payload-pendulum period. This timescale is on the order of seconds and is shorter than the latency typically associated with fault detection, classification, and controller reconfiguration.

This paper studies a decentralized feed-forward choice for this setting. Each drone $i$ sets its altitude thrust feed-forward equal to its measured cable tension,
\begin{equation}
  T_i^{\mathrm{ff}}(t) = T_i(t).
\end{equation}
When a cable severs, the resulting load redistribution is reflected in the local tension measurement. The controller maps this change into a thrust adjustment at the next control update using local information only, rather than requiring a fault flag, inter-agent communication, or an online change of control law.

We analyze the resulting closed loop as a hybrid system with discrete structural transitions and continuous-time dynamics. The feed-forward term contributes the bounded jump property used in the hybrid practical input-to-state-stability argument, subject to the stated actuator, slack-excursion, and dwell-time assumptions. The analysis yields an explicit recovery envelope with exponential decay and per-fault-cycle contraction.

The controller is evaluated in high-fidelity multibody simulation across several fault scenarios. The results show recovery from single and multiple cable severances within one payload-pendulum period for the simulated mission family. Feed-forward ablations indicate that the tension feed-forward term is the main recovery mechanism in the tested baseline cascade; the additional extension layers do not materially change the reported trajectories in the regimes considered.

Related work falls into four categories: centralized and distributed cooperative transport, geometric tracking and slung-load control, $L_1$ adaptive control, and fault-tolerant control.

\textit{Centralized and distributed cooperative transport} uses force-allocation architectures~\cite{Michael2011,Fink2011,Tagliabue2019} and distributed predictive variants~\cite{WangOng2017,LiLoianno2023,PallarLi2025,SuBhowmick2023}. These approaches generally assume synchronized global-state information or coordinated allocation, whereas the cable-severance case considered here changes the constraint set at the actuator timescale. \textit{Geometric tracking}~\cite{LeeLeokMcClamroch2010}, \textit{flatness-based slung-load}~\cite{Sreenath2013}, and $\mathrm{SE}(3)$ \textit{payload control}~\cite{Lee2018} usually assume a fixed constraint set; multi-vehicle extensions~\cite{LeeSreenathKumar2013,SharmaSundaram2023} often reintroduce global-state exchange. \textit{$L_1$ adaptive control} provides matched-disturbance rejection with transient bounds~\cite{CaoHovakimyan2008,Hovakimyan2010book}; its application to cooperative cable-fault settings under the information pattern used here remains open~\cite{WuCheng2025}.

The fault-tolerant control literature is commonly divided into active architectures, which reconfigure after fault diagnosis~\cite{Blanke2015,ZhangJiang2008}, and passive architectures~\cite{Patton1997,Blanke2015}, which maintain robustness without explicit reconfiguration\cite{LiuSuspensionFailure2021}. In cooperative transport, observer-based disturbance estimation with dynamic load reassignment addresses cable failures through detection and reallocation. Other methods, such as passive force control and leader-follower force estimation~\cite{NoCommTransport2021}, reduce the need for inter-UAV exchange but typically assume a fixed number of cables. Recent work has also explored structured representations of system behavior and constraint prioritization as alternatives to reactive or detection-based approaches.~\cite{rulebooks}

Existing approaches therefore either use some form of diagnosis and reconfiguration or analyze systems with a fixed constraint topology. The present formulation addresses the case in which the controller must remain decentralized while the cable set changes abruptly. The tension-to-thrust feed-forward term is not introduced as a new force-feedback principle; rather, the contribution is to analyze this term as part of a hybrid closed loop and to state the conditions under which it yields bounded post-fault recovery.

Measured force feedback and passive impedance/admittance control are mature techniques for modulating stiffness or damping in fixed-topology systems. Cable severance differs from these settings because it changes the cardinality and topology of the active constraint set. In this paper, the measured tension term is analyzed together with the discrete structural transition, rather than as a force gain on a fixed plant. The hybrid practical-ISS result links the feed-forward term to post-fault recovery through the payload-pendulum timescale and the dwell hypothesis, $\tau_d \ge \tau_{\mathrm{pend}}$.

The information pattern is also central to the formulation. Active fault-tolerant pipelines detect, classify, and reconfigure after a component failure. For the severance events considered here, that sequence can be slow relative to the payload-pendulum period. The proposed controller instead uses the local tension change that follows severance and does not use peer-state exchange or centralized diagnosis. This choice limits global optimization, such as ideal load sharing among survivors, so the resulting stability statement is conditional on explicit admissibility assumptions (H1--H3), which are checked in the simulation campaign.

Three contributions organize the results for decentralized fault-tolerant control in cooperative aerial transport:

First, we introduce a local tension-to-thrust feed-forward term, $T_i^{\mathrm{ff}} = T_i$, for cable-severance tolerance in cooperative aerial transport. The term maps the post-severance redistribution of cable tension into a thrust correction using local measurements. A reduced-order taut-cable model with explicit error bounds (\cref{thm:reduction}) supplies the plant model used in the analysis. (C1, \cref{sec:reduction-and-reference})

Second, we prove a conditional hybrid practical input-to-state-stability result for the resulting decentralized closed loop. The proof combines a slack-excursion-bounded taut-cable reduction, inter-fault ISS decay, bounded fault jumps, and dwell-cycle contraction under explicit actuator and dwell-time conditions, yielding a recovery envelope with exponential decay and per-fault-cycle contraction $\rho \in (0,1)$. (C2, \cref{sec:stability})

Third, we validate the mechanism in high-fidelity multibody simulation under single- and dual-cable severance. Feed-forward ablations show that disabling $T_i^{\mathrm{ff}} = T_i$ inflates cruise RMSE by $34$--$39\%$ and peak post-fault sag by $3.6$--$4.0\times$, supporting attribution of the recovery behavior to the feed-forward term within the tested baseline architecture. Extension layers ($L_1$, MPC, reshape) are also evaluated in the regimes considered. (C3, \cref{sec:simulation})

%% file: Section_II_Problem_Statement.tex
We study cooperative aerial transport under unannounced cable-severance faults and formalize the plant, fault model, admissible controller class, reference trajectory, taut-cable reduction, and formal problem. The system consists of:

\paragraph{Vehicles, payload, and state.} A team of $N$ quadrotor UAVs, indexed $i\in\mathcal{N}\triangleq\{1,\dots,N\}$, operates in an inertial world frame $\mathcal{W}$. Drone $i$ has mass $m_i$, diagonal inertia $\mat J_i$, position $\vect p_i$, velocity $\vect v_i$, attitude $\mat R_i\in\SOthree$, and angular velocity $\vect\omega_i$, with per-drone state $x_i\triangleq(\vect p_i,\vect v_i,\mat R_i,\vect\omega_i)$. A point-mass payload of mass $m_L$ has state $x_L\triangleq(\vect p_L,\vect v_L)$; payload rotation is neglected.

\paragraph{Cables and complementarity condition.} Each drone $i$ is connected to the payload by a flexible cable. Cable $i$ exerts a world-frame force $\vect T_i^{\mathcal{W}} \in \R^3$ on the payload-side attachment point and an equal and opposite reaction on the drone-side attachment point. The cable transmits tensile force only when taut; its scalar tension is
\begin{equation}
 T_i \;\triangleq\; \bigl\|\vect T_i^{\mathcal{W}}\bigr\|
     \;\ge\; 0,
  \label{eq:tension-def}
\end{equation}
with the unilateral constraint $T_i = 0$ whenever the cable is slack. The complementarity condition
\begin{equation}
 T_i \;\ge\; 0, \quad
  \ell_i \;-\; L \;\le\; 0, \quad
 T_i\,\bigl(\ell_i - L\bigr) \;=\; 0,
  \label{eq:complementarity}
\end{equation}
where $\ell_i = \|\vect p_{a, i} - \vect p_{b, i}\|$ is the cable's instantaneous length, $\vect p_{a, i}$ and $\vect p_{b, i}$ are the drone- and payload-side attachment points, and $L$ is the rest length, captures the unilateral tensile constraint for an idealized inextensible cable. In the simulator each cable is a Kelvin--Voigt spring ($L = \SI{1.25}{\metre}$, $k_s = \SI{25000}{\newton\per\metre}$); the elastic correction to the inextensible model is $O(\delta)$ and is absorbed in \cref{thm:reduction}.

\paragraph{Equations of motion.} The drone translational dynamics
\begin{equation}
 m_i\,\dot{\vect v}_i
    = f_i\,\mat R_i\,\hat{\vect e}_3
    - m_i g\,\hat{\vect e}_3
    - \vect T_i^{\mathcal{W}}
    + \vect F_{\mathrm{w},i}(t)
  \label{eq:drone-eom}
\end{equation}
combine thrust $f_i\in[f_{\min},f_{\max}]$, gravity, cable reaction, and wind. Standard rigid-body rotational dynamics $\mat J_i\dot{\vect\omega}_i+\vect\omega_i\times\mat J_i\vect\omega_i=\vect\tau_i$, $\dot{\mat R}_i=\mat R_i[\vect\omega_i]_\times$ govern attitude. The payload obeys
\begin{equation}
 m_L\,\dot{\vect v}_L
    = -m_L g\,\hat{\vect e}_3
    + \sum_{i\in\mathcal{S}(t)} \vect T_i^{\mathcal{W}}
    + \vect F_{\mathrm{w},L}(t),
  \label{eq:payload-eom}
\end{equation}
with $\mathcal{S}(t)\subseteq\mathcal{N}$ the surviving-cable set. Per-drone inputs are $u_i=(f_i,\vect\tau_i)\in\mathcal{U}_i\triangleq[f_{\min},f_{\max}]\times[-\tau_{\max},\tau_{\max}]^3$, the plant state $x=(x_{\mathcal{N}},x_L)$ lives in $\mathcal{X}$, and the admissible disturbance set is
\begin{equation}
\label{eq:disturbance-set}
  \mathcal{D} \triangleq \bigl\{ d(\cdot) :
    \|\vect F_{\mathrm{w},i}(t)\|,\; \|\vect F_{\mathrm{w},L}(t)\|
      \le W_{\max},\;\forall i,t \bigr\},
\end{equation}
with $W_{\max}=\SI{1}{\newton}$; the Dryden realization used in V2--V6 stays inside this bound (\cref{tab:sim-params}).

\input{fig_physical_setup.tex}

At a time $t_k^\star$, cable $i_k^\star \in \mathcal{N}$ severs instantaneously. This structural fault event permanently removes cable $i_k^\star$ from the constraint set:
\begin{equation}
 T_{i_k^\star}(t) \;=\; 0
  \qquad \forall\, t \ge t_k^\star.
  \label{eq:fault-tension-zero}
\end{equation}
The active cable set transitions discretely:
\begin{equation}
  \mathcal{S}(t_k^{\star,+}) \;=\; \mathcal{S}(t_k^{\star,-})
    \;\setminus\; \{i_k^\star\},
  \label{eq:S-transition}
\end{equation}
reducing the cardinality from $|\mathcal{S}|$ to $|\mathcal{S}| - 1$. For an elastic cable, the kinematic state $(\vect p_i, \vect v_i, \mat R_i, \vect\omega_i, \vect p_L, \vect v_L)$ is continuous across severance; only the ODE right-hand side changes discontinuously at $t_k^\star$.

We consider a sequence of $F \le N-2$ fault events at times $\{t_k^\star\}_{k=1}^F$ with $t_1^\star < \cdots < t_F^\star$. The bound $F \le N-2$ keeps at least two cables intact; actuator margins limit the demonstrated scenario ($N = 5$, $m_L = \SI{10}{\kilo\gram}$) to $F \le 2$. The inter-fault interval $\tau_{d,k} \triangleq t_{k+1}^\star - t_k^\star$ is the \emph{dwell time}. We assume each fault is separated from the previous by at least one payload-pendulum period:
\begin{equation}
  \tau_{d,k} \;\ge\; \tau_{\mathrm{pend}} \;\triangleq\;
    2\pi\sqrt{L/g}, \qquad k = 1, \dots, F-1,
  \label{eq:dwell-assumption}
\end{equation}
reflecting the physical requirement that load redistribution from fault $k$ settle before fault $k+1$. At $L = \SI{1.25}{\metre}$ and $g = \SI{9.81}{\metre\per\second\squared}$, $\tau_{\mathrm{pend}} \approx \SI{2.24}{\second}$. The admissible fault-schedule set is
\begin{align}
\label{eq:fault-schedule-set}
  \mathcal{F} \;\triangleq\;
    \Bigl\{
      \{(t_k^\star, i_k^\star)\}_{k=1}^F :
 &F \le N-2,\;\\
 &t_{k+1}^\star - t_k^\star \ge \tau_{\mathrm{pend}},\;
 i_k^\star \in \mathcal{S}(t_k^{\star,-})
    \Bigr\},\nonumber
\end{align}
$F = 0$ corresponds to the no-fault case.

\paragraph{Admissible controller class.} No fault flag, annunciation, or detection signal is transmitted; the fault manifests through a rise in $T_i(t)$ and a transient payload-velocity deviation, both in $\mathcal{I}_i(t)$. Field-deployable avionics impose a strict \emph{admissible information pattern}. Drone $i$'s controller $\pi_i$ is restricted to the information set
\begin{equation}
  \mathcal{I}_i(t) \;\triangleq\;
    \bigl\{x_i(t),\; T_i(t),\; \vect v_L(t),\;
      \vect p_L^d(\cdot),\; \vect v_L^d(\cdot)\bigr\},
  \label{eq:info-set}
\end{equation}
comprising the drone's inertial state $x_i$, its rope-tension measurement $T_i$, the payload velocity $\vect v_L$ (measured via a co-located downward-facing optical-flow sensor, no drone-to-drone communication), and a pre-distributed reference $({\vect p}_L^d, {\vect v}_L^d)$ loaded before flight. Payload position $\vect p_L$ is deliberately excluded; only velocity enters the closed loop, via the anti-swing slot shift of \cref{sec:proposed}. Each $\pi_i$ is a causal feedback map,
\begin{equation}
 u_i(t) \;=\; \pi_i\!\left(
    \left.\mathcal{I}_i\right|_{[0,t]}
  \right) \in \mathcal{U}_i,
  \label{eq:controller-map}
\end{equation}
possibly with a finite-dimensional internal state, but never with access to information outside \eqref{eq:info-set}. The admissible controller class is denoted $\Pi_{\mathrm{loc}}$. $\mathcal{I}_i(t)$ excludes peer state, peer tension, fault flags, and run-time inter-drone packets; each $\pi_i$ uses only co-located sensors and a pre-flight plan.

\subsection{Reference Trajectory and Taut-Cable Reduction}
\label{sec:reduction-and-reference}

The desired payload motion is a three-dimensional lemniscate trajectory with a sinusoidal altitude variation,
\begin{align}
  \label{eq:lemniscate}
  \vect p_L^d(t) \;=\;
    \begin{bmatrix}
      \dfrac{a\cos\phi(t)}{1 + \sin^2\!\phi(t)} \\[6pt]
      \dfrac{a\sin\phi(t)\cos\phi(t)}{1 + \sin^2\!\phi(t)} \\[6pt]
 z_0 + h_z\sin(\phi(t))
    \end{bmatrix},
  \phi(t) = \frac{2\pi(t - t_0)}{T_{\mathrm{ref}}},
\end{align}
where $a = \SI{3.0}{\metre}$ is the horizontal scale, $T_{\mathrm{ref}} = \SI{12}{\second}$ the period, $z_0 = \SI{3.0}{\metre}$ the cruise altitude, and $h_z = \SI{0.35}{\metre}$ the altitude excursion amplitude. These parameters yield a peak centripetal load of $\approx 0.25g$ at the figure-8 tips.

Each rope of rest length $L$ is discretized into $N_b = 8$ identical Kelvin--Voigt beads of mass $m_{\mathrm{bead}}$ connected by $N_{\mathrm{seg}} = N_b + 1 = 9$ segments of stiffness $k_s$ and damping $c_s$; $N_{\mathrm{seg}}$ is reserved for rope segments and is distinct from the surviving-cable cardinality $N_s(t) = |\mathcal{S}(t)| = N - F(t)$ used in \cref{sec:stability}. The distributed dynamics evolve on a fast timescale,
\begin{equation}
  \tau_{\mathrm{rope}}
    \;=\; 2\pi\sqrt{m_{\mathrm{bead}}/k_s}
    \;\approx\; \SI{6.3}{\milli\second},
  \label{eq:tau-rope}
\end{equation}
separated from the payload-pendulum timescale $\tau_{\mathrm{pend}} \approx \SI{2.24}{\second}$ by $\delta = \tau_{\mathrm{rope}} / \tau_{\mathrm{pend}} \approx 0.003$. On the slow manifold, the net tensile force at the drone-side attachment equals the scalar $\xi_i = k_{\mathrm{eff}}(\ell_i - L)^+$, with effective stiffness
\begin{equation}
 k_{\mathrm{eff}}
    \;=\; \frac{k_s}{N_{\mathrm{seg}}}
    \;=\; \frac{\SI{25000}{\newton\per\metre}}{9}
    \;\approx\; \SI{2778}{\newton\per\metre}.
  \label{eq:keff}
\end{equation}
The six missions (V1--V6) keep the system inside $\Omega_\tau^{\mathrm{dwell}}$ with margin (worst slack run $\SI{35.8}{\milli\second}$, worst duty cycle $1.66\%$, \cref{tab:domain-audit}). The reduction replaces the $3N_b$-dimensional bead-chain state with the scalar $\xi_i$; the per-instant taut error is $O(\delta)$ and slack excursions add $O(\eta_{\max})$.

\begin{theorem}[Taut-cable reduction]
\label{thm:reduction}
On the slack-excursion-bounded domain
\begin{align}
\label{eq:omega-dwell}
  \Omega_\tau^{\mathrm{dwell}} \;\triangleq\;
    \bigl\{ x \in \mathcal{X} :
      &\text{every slack run} \le \tau_{\mathrm{slack,max}},\;\\
      &\text{duty cycle} \le \eta_{\max} \bigr\},\nonumber
\end{align}
with $\tau_{\mathrm{slack,max}} = \SI{40}{\milli\second}$ and $\eta_{\max} = 2.5\%$, the distributed Kelvin--Voigt rope dynamics are approximated by the lumped scalar model with closed-loop shape-mode deviation
\begin{equation}
  \|T_i^{\mathrm{true}} - T_i^{\mathrm{qs}}\|
    \;\le\; C_1\,\delta + C_2\,\eta_{\max},
  \label{eq:reduction-error}
\end{equation}
for constants $C_1, C_2 > 0$ depending on the bead-chain parameters $(m_{\mathrm{bead}}, c_s, k_s, N_b)$ and the wind bound $W_{\max}$, where $T_i^{\mathrm{qs}} \triangleq k_{\mathrm{eff}}(\ell_i(t) - L)^+$ is the\emph{ rope-specific} quasi-static tension evaluated at the \emph{same} current chord length $\ell_i(t)$ as $T_i^{\mathrm{true}}$.  This bound is on the shape-mode deviation of rope $i$ from its own quasi-static approximation; it is distinct from the inter-rope tension asymmetry $\varepsilon_i = T_i - T_{\mathrm{eff}}$ reported in the simulation assessment (\cref{sec:VI-reduction}), which is driven by formation geometry and can be $O(10)$\,N. The reduced-order model is therefore admissible as the plant for the stability analysis of \cref{sec:stability}.
\end{theorem}

\begin{proof}[Sketch]
The proof has two parts corresponding to the two terms in \eqref{eq:reduction-error}.

\noindent\textbf{Part 1 ($C_1\delta$, taut intervals).} Let $\tilde{\mathbf{q}}\in\R^{N_b}$ collect bead deviations from the quasi-static slow-manifold equilibrium $\bar{\mathbf{q}}(t)$ along the rope chord. Linearization yields the shape-mode dynamics $M_b\ddot{\tilde{\mathbf{q}}}+C_b\dot{\tilde{\mathbf{q}}}+K_b\tilde{\mathbf{q}}=-M_b\ddot{\bar{\mathbf{q}}}$ with $C_b,K_b$ the chain damping and stiffness matrices. Differentiating the shape-mode energy $W=\tfrac12\dot{\tilde{\mathbf{q}}}^\top M_b\dot{\tilde{\mathbf{q}}}+\tfrac12\tilde{\mathbf{q}}^\top K_b\tilde{\mathbf{q}}$ and applying Young's inequality gives $\dot W\le -\alpha_b W + D_b$ with $\alpha_b=\lambda_{\min}(C_b)/m_{\mathrm{bead}}$ and $D_b=O(\delta^2)$ via the Tikhonov scaling $\|\ddot{\bar{\mathbf{q}}}\|=O(\delta)$. Gronwall yields $W^\infty=D_b/\alpha_b=O(\delta^2)$; the drone-side tension deviation $T_i^{\mathrm{true}}-T_i^{\mathrm{qs}}=k_s\mathbf{e}_1^\top\tilde{\mathbf{q}}$ then satisfies $\|T_i^{\mathrm{true}}-T_i^{\mathrm{qs}}\|\le C_1\delta\approx\SI{0.072}{\newton}$ after the $\sim\SI{6}{\milli\second}$ transient. The slowest mode has $\zeta=1.2$, $s_1\approx-\SI{538}{\second^{-1}}$.

\noindent\textbf{Part 2 ($C_2\eta_{\max}$, slack intervals).} On any slack run of duration $\le\tau_{\mathrm{slack,max}}$, the chord shortens by $\Delta\ell_{\mathrm{sag}}$ and the drone-side end segment must re-elongate on re-tension, producing a re-tension overshoot bounded by $k_{\mathrm{eff}}\Delta\ell_{\mathrm{sag}}$. Aggregating over the slack-duty budget gives $C_2\approx\SI{0.22}{\newton}$. At the worst observed $\eta_{\max}=1.66\%$, the combined bound is $C_1\delta+C_2\eta_{\max}\approx\SI{0.076}{\newton}$ per rope ($<0.4\%$ of nominal per-drone tension). Both constants depend only on $(m_{\mathrm{bead}},c_s,k_s,N_b)$ and $W_{\max}$. Detailed Gronwall arithmetic and shape-mode constants are in the supplementary material.
\end{proof}

\Cref{tab:domain-audit} shows that all six missions remain inside $\Omega_\tau^{\mathrm{dwell}}$ with margin on every gate: worst slack run $\SI{35.8}{\milli\second}$ ($\SI{4.2}{\milli\second}$ headroom), duty cycle $1.66\%$ ($0.85$~pp headroom), QP-transition fraction $0.10\%$. Accordingly, the reduction is used on the demonstrated operating regime rather than as an all-taut approximation.

The per-drone controller and the lemma stack in \cref{sec:stability} are structurally independent of $N$. The team size enters only through the post-fault redistributed load
\begin{equation}
  T^{\mathrm{post}}(N,F) = \frac{m_L g}{N-F},
   \label{eq:tension-postfault-N}
\end{equation}
yielding the actuator-margin condition
\begin{equation}
  \frac{m_L g}{N-F} \le \kappa_{\mathrm{act}}\,f_{\max},
   \label{eq:actuator-envelope}
\end{equation}
with $\kappa_{\mathrm{act}}\in(0,1)$ reserving headroom for tilt projection, anti-swing residual, and dynamic centripetal load. At canonical $\kappa_{\mathrm{act}}=0.82$, $f_{\max}=\SI{150}{\newton}$, and the demonstrated $(N,m_L,F)=(5,\SI{10}{\kilo\gram},2)$, \eqref{eq:actuator-envelope} reads $\SI{32.7}{\newton}\le\SI{123}{\newton}$, i.e.\ $27\%$ of envelope. The stability argument of \cref{sec:stability} remains unchanged at any $(N,m_L,F)$ satisfying \eqref{eq:actuator-envelope}.
\label{sec:formal-problem}
% ----------------------------------------------------------------

With the plant, fault model, reference trajectory, and reduction in place, we state the formal problem. Unilateral cable coupling \eqref{eq:complementarity}, discrete structural fault transitions \eqref{eq:fault-tension-zero}--\eqref{eq:S-transition}, admissible disturbances \eqref{eq:disturbance-set}, and the local information pattern \eqref{eq:info-set} together yield a hybrid decentralized control problem outside standard cooperative-transport formulations. The design variable is not a centralized control sequence but a family of local causal feedback laws. The problem, posed on the reduced-order model supported by \cref{thm:reduction}, is as follows.

\begin{definition}[Fault-tolerant decentralized aerial transport]
\label{def:problem}
Given the plant \eqref{eq:drone-eom}--\eqref{eq:payload-eom}, the complementarity constraint \eqref{eq:complementarity}, the mission reference \eqref{eq:lemniscate}, all $d(\cdot)\in\mathcal{D}$ \eqref{eq:disturbance-set}, and all $\sigma\in\mathcal{F}$ \eqref{eq:fault-schedule-set}, find a causal controller family $\{\pi_i\}_{i=1}^N\in\Pi_{\mathrm{loc}}^N$ such that the closed loop satisfies the following.
\begin{enumerate}[label=\textbf{P\arabic*},leftmargin=2.0em]
  \item \textbf{Nominal tracking.} The payload tracking error $e_L(t)\triangleq\vect p_L(t)-\vect p_L^d(t)$ is ultimately practically bounded,
    $\limsup_{t\to\infty}\|e_L(t)\|\le\bar e_\infty+\bar e_0(W_{\max})$, with $\bar e_0\in\mathcal{K}$.
    \label{eq:nominal-practical-bound}
  \item \textbf{Passive fault absorption.} After any fault $t_k^\star$ obeying \eqref{eq:dwell-assumption}, $e_L$ recovers to a bounded envelope,
    \begin{align}
    \label{eq:recovery-bound}
      \|e_L(t)\| \le
        &\beta\bigl(\|e_L(t_k^{\star,-})\|,\,t-t_k^\star\bigr)
         + \Gamma_f\,e^{-\lambda(t-t_k^\star)}\\
        &+ \bar e_\infty + \bar e_0(W_{\max}),\quad t\ge t_k^\star,\nonumber
    \end{align}
    for some $\beta\in\mathcal{KL}$ and constants $\Gamma_f,\lambda>0$, with per-fault-cycle contraction $V(x(t_{k+1}^{\star,-}))\le\rho V(x(t_k^{\star,-}))+c$ for $\rho\in(0,1)$ and explicit $c>0$.
  \item \textbf{Actuator feasibility.} $f_i(t)\in[f_{\min},f_{\max}]$ and $\vect\tau_i(t)\in[-\tau_{\max},\tau_{\max}]^3$ for all $t\ge 0$ and all admissible fault sequences, with non-zero margin.
\end{enumerate}
\end{definition}

The discrete drop in $|\mathcal{S}|$ at each severance, the non-smooth taut-to-slack transitions inside $\Omega_\tau^{\mathrm{dwell}}$, and the local information constraint together require a hybrid Lyapunov argument on the tension mismatch
\begin{equation}
  \vect\varepsilon_T(t) \;\triangleq\; \vect T(t) - \vect T^{\mathrm{qs}}(t),
  \label{eq:tension-error-state}
\end{equation}
developed in \cref{sec:stability}.

%% file: fig_physical_setup.tex
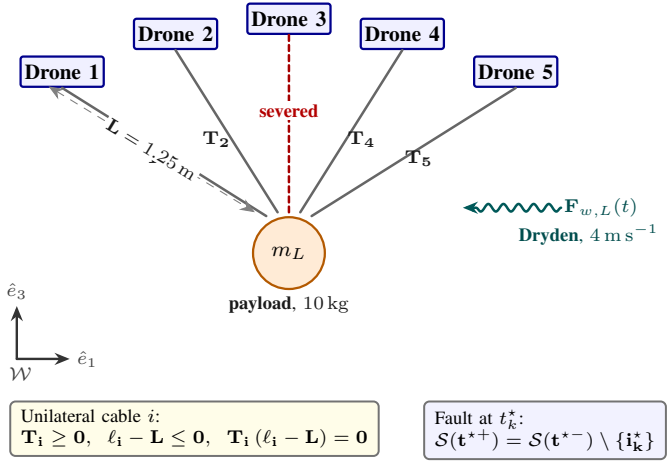
\begin{figure}[t]
\centering
\begin{tikzpicture}[
            line cap=round,
            line join=round,
    >={Stealth[length=2mm,width=1.6mm]},
            drone/.style={draw=blue!55!black, line width=0.85pt,
                                                      fill=blue!6, rounded corners=1pt,
                                                      minimum width=11mm, minimum height=4mm,
                                                      font=\footnotesize, inner sep=1pt},
            payload/.style={draw=orange!70!black, line width=0.85pt,
                                                            fill=orange!16, circle, minimum size=9.5mm,
                                                            inner sep=0pt, font=\footnotesize\bfseries},
            cable/.style={line width=0.95pt, black!60},
            cablesev/.style={line width=0.95pt, red!65!black,
                                                             dash pattern=on 2.2pt off 1.6pt},
            wind/.style={->, line width=0.95pt, teal!70!black, decorate,
                 decoration={snake, amplitude=0.5mm, segment length=2mm,
                             post length=1mm}},
            framelab/.style={font=\scriptsize\itshape, text=black!75},
            lab/.style={font=\scriptsize, inner sep=1pt, text=black!90},
            callout/.style={draw=black!55, line width=0.75pt,
                                                            rounded corners=2pt, font=\scriptsize,
                                                            align=left, inner xsep=4pt, inner ysep=3pt}
  ]

  %====== payload at origin ======
  \node[payload] (L) at (0,0) {$m_L$};
      \node[lab, below=1pt of L] {\textbf{payload}, \SI{10}{\kilo\gram}};

  %====== five drones evenly spaced overhead ======
      \node[drone] (D1) at (-3.0, 2.4) {\textbf{Drone 1}};
      \node[drone] (D2) at (-1.5, 2.9) {\textbf{Drone 2}};
      \node[drone] (D3) at ( 0.0, 3.1) {\textbf{Drone 3}};
      \node[drone] (D4) at ( 1.5, 2.9) {\textbf{Drone 4}};
      \node[drone] (D5) at ( 3.0, 2.4) {\textbf{Drone 5}};

  %====== cables: 4 taut, 1 severed (drone 3) ======
  \draw[cable] (D1.south) --
        node[lab, left=-1pt, pos=0.55] {$\mathbf{T_1}$} ($(L.north)+(-0.3, 0)$);
  \draw[cable] (D2.south) --
        node[lab, left=-1pt, pos=0.55] {$\mathbf{T_2}$} ($(L.north)+(-0.15, 0.05)$);
  \draw[cablesev] (D3.south) --
        node[lab, midway, above=1pt, fill=white, text=red!70!black] {\textbf{severed}}
        ($(L.north)+(0, 0.05)$);
  \draw[cable] (D4.south) --
        node[lab, right=-1pt, pos=0.55] {$\mathbf{T_4}$} ($(L.north)+(0.15, 0.05)$);
  \draw[cable] (D5.south) --
        node[lab, right=-1pt, pos=0.55] {$\mathbf{T_5}$} ($(L.north)+(0.3, 0)$);

  %====== world frame (bottom left, away from cables) ======
  \draw[->, line width=0.9pt, black!80] (-3.6,-1.4) -- ++(0.7,0)
        node[framelab, right=-1pt] {$\hat{e}_1$};
  \draw[->, line width=0.9pt, black!80] (-3.6,-1.4) -- ++(0,0.7)
        node[framelab, above=-1pt] {$\hat{e}_3$};
  \node[framelab] at (-3.55,-1.6) {$\mathcal{W}$};

  %====== Dryden wind arrow (right side) ======
  \draw[wind] (3.6, 0.6) -- ++(-1.3, 0.0);
  \node[lab, teal!60!black, anchor=west] at (3.6, 0.6) {$\mathbf{F}_{w,L}(t)$};
      \node[lab, teal!60!black, anchor=west] at (3.0, 0.25)
        {\textbf{Dryden}, $\mathbf{\SI{4}{\metre\per\second}}$};

  %====== rope length marker on outermost cable ======
  \draw[<->, dashed, black!50]
        ($(D1.south)+(-0.18, 0.02)$) -- ($(L.north)+(-0.45, 0.05)$)
        node[lab, midway, sloped, fill=white, inner sep=1pt]
        {$\mathbf{L = \SI{1.25}{\metre}}$};

  %====== complementarity callout (placed BELOW everything) ======
  \node[callout, fill=yellow!10, anchor=north west]
       at (-3.7, -1.95)
       {Unilateral cable $i$:\\
        $\mathbf{T_i \ge 0,\;\; \ell_i - L \le 0,\;\; T_i\,(\ell_i - L) = 0}$};

  %====== survivor set transition (placed BELOW, right) ======
  \node[callout, fill=blue!5, anchor=north east]
       at (5.0, -1.95)
       {Fault at $t_k^\star$:\\
        $\mathbf{\mathcal{S}(t^{\star+}) = \mathcal{S}(t^{\star-}) \setminus \{i_k^\star\}}$};

\end{tikzpicture}
\caption{Cooperative transport configuration and fault model.
Five quadrotors carry a \SI{10}{\kilo\gram} point-mass payload through
unilateral cables of rest length $L = \SI{1.25}{\metre}$. The
illustration shows one severed cable, the payload wind load,
the world frame, the cable complementarity constraint, and the
survivor-set update at a fault. Each drone uses only its own state,
its own scalar tension $T_i$, and the locally observed payload velocity
$\vect{v}_L$; no peer drone or peer-cable signal is exchanged.}
\label{fig:physical-setup}
\end{figure}

%% file: Section_III_Proposed_Method.tex
Each drone runs a three-layer cascade: an outer-loop PD with anti-swing slot shift, a single-step convex QP on the tilt-and-thrust envelope, and an inner-loop attitude PD. The only nonstandard baseline element is the local feed-forward identity $T_i^{\mathrm{ff}}=T_i$, which maps tension redistribution after severance into immediate thrust compensation.

\input{fig_controller_cascade.tex}

Each drone is pre-assigned a hovering slot displaced vertically above the payload reference point and laterally by a fixed formation offset,
\begin{equation}
  \vect p_{\mathrm{slot},i}(t)
    \;=\; \vect p_L^d(t) + \boldsymbol\delta_i
        + d_{\mathrm{rope}}\,\hat{\vect e}_3,
  \label{eq:slot-pos}
\end{equation}
where $\boldsymbol\delta_i \in \R^3$ is the formation offset (overridable at run time), $d_{\mathrm{rope}} = \SI{1.25}{\metre}$ is the vertical slot elevation matching $L$, and $\hat{\vect e}_3$ is the upward unit vector. The slot velocity is $\vect v_{\mathrm{slot},i}(t) = \vect v_L^d(t)$. To damp the payload pendulum passively, we shift the slot relative to the payload's horizontal velocity. With $\vect v_{L,\perp} \triangleq [v_{L,x}, v_{L,y}, 0]^{\top}$, the horizontal swing correction is
\begin{equation}
  \vect s_i(t) =
    \begin{cases}
      -k_{\mathrm{swing}}\,\vect v_{L,\perp}
        & \text{if }
          k_{\mathrm{swing}}\|\vect v_{L,\perp}\| \le s_{\max}, \\[2pt]
      -s_{\max}\,\hat{\vect v}_{L,\perp}
        & \text{otherwise,}
    \end{cases}
  \label{eq:swing-shift}
\end{equation}
where $\hat{\vect v}_{L,\perp}$ is the unit vector of $\vect v_{L,\perp}$. The shifted slot position is $\vect p_{\mathrm{slot},i}^{\mathrm{dyn}}(t) = \vect p_{\mathrm{slot},i}(t) + \vect s_i(t)$; therefore, the position and velocity errors are
\begin{align}
    \label{eq:slot-error}
  \vect e_p(t) &= \vect p_{\mathrm{slot},i}^{\mathrm{dyn}}(t)
                 - \vect p_i(t),\\
  \vect e_v(t) &= \vect v_{\mathrm{slot},i}(t) - \vect v_i(t).\nonumber
\end{align}

The outer loop computes a target body acceleration,
\begin{align}
\label{eq:a-target}
  \vect a_{\mathrm{target}}(t)
    =& \underbrace{%
        \begin{bmatrix}
          K_p^{xy}\, e_{p,x} + K_d^{xy}\, e_{v,x} \\[2pt]
          K_p^{xy}\, e_{p,y} + K_d^{xy}\, e_{v,y} \\[2pt]
          K_p^{z}\;\; e_{p,z} + K_d^{z}\;\; e_{v,z}
        \end{bmatrix}}_{\vect a_{\mathrm{track}}}\\
    &+\; w_{\mathrm{swing}}
      \begin{bmatrix}
        -k_{\mathrm{swing}}\, v_{L,x} \\[2pt]
        -k_{\mathrm{swing}}\, v_{L,y} \\[2pt]
        0
      \end{bmatrix}.\nonumber
\end{align}
The altitude gain $K_p^z = 100$ exceeds the horizontal gain $K_p^{xy} = 30$ because vertical rope forces dominate the disturbance budget; the formal pole and Lyapunov-matrix analysis is in \cref{lem:altitude-iss}. When $L_1$ is active, its altitude correction $u_{\mathrm{ad}}(t)$ is added to the $z$-component of $\vect a_{\mathrm{target}}$ between the outer PD and the QP (\cref{sec:extensions}); horizontal channels are unmodified.

The target acceleration $\vect a_{\mathrm{target}}$ may lie outside the tilt-and-thrust actuator envelope. We project onto the feasible set with a single-step convex QP that also regularises large commands:
\begin{subequations}
\label{eq:qp}
\begin{align}
  \vect a_d^\star
    &= \operatorname*{argmin}_{\vect a_d \in \R^3}
      \; w_t\,\bigl\|\vect a_d - \vect a_{\mathrm{target}}\bigr\|^2
      + w_e\,\|\vect a_d\|^2,
    \label{eq:qp-cost} \\
  \text{s.t.}\;\;
    &\abs{a_{d,x}},\;\abs{a_{d,y}}
      \le a_{\mathrm{tilt}} \triangleq g\tan\theta_{\max},
    \label{eq:qp-tilt} \\
    &a_{z,\min}(t) \le a_{d,z} \le a_{z,\max}(t),
    \label{eq:qp-vert}
\end{align}
\end{subequations}
where the vertical bounds incorporate the rope tension feed-forward,
\begin{align}
\label{eq:az-bounds}
  a_{z,\min}(t) &= \frac{f_{\min} - T_i^{\mathrm{ff}}(t)}{m_i} - g,\\
  a_{z,\max}(t) &= \frac{f_{\max} - T_i^{\mathrm{ff}}(t)}{m_i} - g.\nonumber
\end{align}
The vertical interval in \eqref{eq:qp-vert} is non-empty whenever $f_{\min} < f_{\max}$. In the reported simulations, hover feasibility holds and the D1 active-set condition is satisfied at the $0.10\%$ worst QP-transition fraction (\cref{tab:domain-audit}). For the canonical campaigns, the feed-forward is applied from $t=0$:
\begin{equation}
  \boxed{T_i^{\mathrm{ff}}(t) \;=\; T_i(t).}
  \label{eq:Tff}
\end{equation}
A smoothstep pickup ramp $\alpha(\tau) = 3\tau^2 - 2\tau^3$ handles cold-start engagement; it is bypassed in every reported campaign.

Given the QP solution $a_{d,z}^\star$, the commanded thrust is
\begin{equation}
\label{eq:thrust}
  f_i \;=\; \operatorname{clamp}\!\bigl(\,
      m_i(g + a_{d,z}^\star) + T_i^{\mathrm{ff}},\;
      f_{\min},\; f_{\max}\bigr).
\end{equation}
Under the small-tilt regime enforced by \eqref{eq:qp-tilt}, the horizontal acceleration commands map to desired pitch and roll via $\theta^d_{\mathrm{pitch}}=\operatorname{sat}(a_{d,x}^\star/g,\pm\theta_{\max})$, $\theta^d_{\mathrm{roll}}=\operatorname{sat}(-a_{d,y}^\star/g,\pm\theta_{\max})$, tracked by a per-axis Euler-angle PD with gains $(K_p^{\mathrm{att}},K_d^{\mathrm{att}})$ and torques saturated to $[-\tau_{\max},\tau_{\max}]^3$. The yaw channel uses the standard $\SOthree$-projected error $-(R_{21}-R_{12})/2$. \Cref{tab:baseline-params} collects the canonical controller parameters used throughout all simulation campaigns.

\begin{table}[t]
  \centering
  \caption{Proposed controller parameters (canonical values used
  in all simulation campaigns).}
  \label{tab:baseline-params}
  \small
  \renewcommand{\arraystretch}{1.05}
  \begin{tabular}{lrl}
    \toprule
    Parameter & Value & Role \\
    \midrule
    $K_p^{xy},\; K_d^{xy}$
      & $30,\;15$
      & horizontal outer loop \\
    $K_p^{z},\;\; K_d^{z}$
      & $100,\;24$
      & altitude outer loop \\
    $K_p^{\mathrm{att}},\; K_d^{\mathrm{att}}$
      & $25,\;4$
      & attitude inner loop \\
    $k_{\mathrm{swing}}$
      & $0.8$
      & anti-swing gain \\
    $w_{\mathrm{swing}}$
      & $0.3$
      & anti-swing weight in $\vect a_{\mathrm{target}}$ \\
    $s_{\max}$
      & \SI{0.3}{\metre}
      & slot-shift saturation \\
    $d_{\mathrm{rope}}$
      & \SI{1.25}{\metre}
      & slot $z$-offset (nominal rope length) \\
    $w_t,\; w_e$
      & $1.0,\;0.02$
      & QP tracking / effort weights \\
    $\theta_{\max}$
      & \SI{0.6}{\radian}
      & tilt ceiling (${\approx}34^{\circ}$) \\
    $f_{\min},\; f_{\max}$
      & $0,\;\SI{150}{\newton}$
      & thrust envelope \\
    $\tau_{\max}$
      & \SI{10}{\newton\metre}
      & torque saturation \\
    $\bar{T}$
      & \SI{19.62}{\newton}
      & nominal per-drone tension ($m_L g/N$) \\
    \bottomrule
  \end{tabular}
\end{table}

%% file: fig_controller_cascade.tex
\begin{figure*}[t]
\centering
\begin{tikzpicture}[
      line cap=round,
      line join=round,
    >={Stealth[length=2mm,width=1.6mm]},
      blk/.style={draw=blue!55!black, line width=0.85pt,
                        rounded corners=1pt, fill=blue!6,
                        minimum height=11mm, minimum width=22mm,
                        align=center, font=\footnotesize, inner sep=2pt},
      extblk/.style={draw=orange!70!black, line width=0.85pt,
                           rounded corners=1pt, fill=orange!10,
                           minimum height=10mm, minimum width=26mm,
                           align=center, font=\footnotesize, inner sep=2pt},
      sens/.style={draw=green!45!black, line width=0.85pt,
                         rounded corners=1pt, fill=green!8,
                         minimum height=9mm, minimum width=30mm,
                         align=center, font=\footnotesize, inner sep=2pt},
      plant/.style={draw=black!60, line width=0.85pt,
                          rounded corners=1pt, fill=black!4,
                          minimum height=11mm, minimum width=24mm,
                          align=center, font=\footnotesize, inner sep=2pt},
      sum/.style={draw=black!80, line width=0.85pt, circle, inner sep=0.5pt,
                minimum size=4.5mm, font=\scriptsize},
      sig/.style={font=\scriptsize, inner sep=1.2pt, fill=white, text=black!90},
      boxed/.style={draw=black!45, dashed, rounded corners=2pt, inner sep=4pt},
      paneltitle/.style={font=\scriptsize\bfseries, fill=white,
                                 inner sep=1.5pt, text=black!80},
      flow/.style={->, line width=0.95pt, black!85},
      fftap/.style={->, line width=0.95pt, red!70!black},
      senseflow/.style={->, line width=0.95pt, blue!60!black},
      extflow/.style={->, line width=0.95pt, dashed, orange!75!black}
  ]

  % --- Baseline cascade row (y = 0) ---
  \node[sig, font=\scriptsize] (refin) at (-8.6, 0)
        {$\mathbf{p}_L^d,\,\mathbf{v}_L^d$};
      \node[blk] (pd)    at (-6.0, 0) {Outer PD\\anti-swing\\slot shift};
  \node[sum] (sumL1) at (-3.0, 0) {$+$};
      \node[blk] (qp)    at ( 0.0, 0) {Single-step QP\\tilt/thrust\\envelope};
  \node[blk] (att)   at ( 3.4, 0) {Attitude PD\\(ZYX Euler)};
  \node[plant] (plant) at ( 6.7, 0) {Drone $i$ + cable\\(reduced model)};

  % horizontal cascade flow
  \draw[flow] (refin) -- (pd);
  \draw[flow] (pd) -- node[sig, above] {$\mathbf{a}_{\mathrm{target}}$}
        (sumL1);
  \draw[flow] (sumL1) -- (qp);
  \draw[flow] (qp) -- node[sig, above] {$\mathbf{a}_d^\star$} (att);
  \draw[flow] (att) -- node[sig, above] {$f_i,\,\boldsymbol{\tau}_i$}
        (plant);

  % --- Feed-forward identity tap (sumT directly above QP) ---
      \node[sum, draw=red!70!black, fill=red!8] (sumT) at (0, 1.2) {$+$};
  % Identity label: placed ABOVE sumT (out of any tap path);
  % no white fill so it does not look like a small box.
  \node[anchor=south, font=\scriptsize\bfseries, red!70!black,
        inner sep=2pt] at (sumT.north)
        {$\mathbf{T_i^{\mathrm{ff}}\!=\!T_i}$};
  \draw[fftap] (sumT.south) -- (qp.north);

  % --- Sensor band (y = -2.6, centre at x = 1.7) ---
  % x = 1.7 sits in the clear gap between QP (right edge x=1.1)
  % and att (left edge x=2.3), so sense.north runs vertically up
  % to sumT without crossing any block.
  \node[sens] (sense) at (1.7, -2.2)
        {Local sensors\\$x_i,\;\; T_i,\;\; v_L$};

  % Plant -> sensors: |- (vertical down then horizontal left)
  \draw[senseflow] (plant.south) |- (sense.east);

  % Sensors -> PD: horizontal left at y = -2.6, then vertical up
  % at x = -6.0.  Single draw call -> ONE arrowhead at pd.south.
  \draw[senseflow] (sense.west)
        -- node[sig, pos=0.63, above] {$\mathbf{x_i,\,v_L}$}
        (-6.0, -2.2) -- (pd.south);

  % T_i tap (red) — sense.NORTH straight up, then horizontal
  % left to sumT.east.  Single draw call -> ONE arrowhead at
  % sumT.east.  Vertical at x=1.7 clears QP (right edge x=1.1)
  % and att (left edge x=2.3); horizontal at y=1.7 sits well
  % above all baseline blocks (top edges at y ~= 0.55).
  % T_i label placed near the start of the vertical (well below
  % the cascade row), with no fill so it does not look like a
  % small block sitting on the cascade.
  \draw[fftap] (sense.north)
      -- node[anchor=west, font=\scriptsize\bfseries,
                red!70!black, pos=0.20, inner sep=1pt]
            {$\mathbf{T_i}$}
        (sense.north |- sumT.east) -- (sumT.east);

  % --- Bracket the baseline ---
  \begin{scope}[on background layer]
    \node[boxed,
          fit=(refin) (pd) (sumL1) (qp) (att) (plant) (sumT) (sense)]
         (baseline) {};
  \end{scope}
  \node[paneltitle, anchor=south west]
        at ([xshift=6pt]baseline.north west)
        {Baseline controller};

  % --- Extension band (y = -4.6) ---
  \node[extblk] (reshape) at (-6.0, -4.2)
        {Reshape supervisor\\$\lceil\log_2 N\rceil$-bit broadcast};
  \node[extblk] (l1) at (-3.0, -4.2)
        {$L_1$ adaptive\\altitude augmentation};
  \node[extblk] (mpc) at ( 0.0, -4.2)
        {Receding-horizon MPC\\replaces QP};

  % All three extension injection arrows are strictly vertical.
  \draw[extflow] (reshape.north)
        -- node[sig, pos=0.78, right=2pt]
              {$\mathbf{\boldsymbol{\delta}_i^{\mathrm{ovr}}}$}
        (pd.south);
  \draw[extflow] (l1.north)
        -- node[sig, pos=0.78, right=2pt] {$u_{\mathrm{ad}}$}
        (sumL1.south);
  \draw[extflow] (mpc.north)
        -- node[sig, pos=0.78, right=2pt] {\textbf{replaces QP}}
        (qp.south);

  % --- Bracket the extension layers ---
  \begin{scope}[on background layer]
    \node[boxed, fit=(reshape) (l1) (mpc), draw=orange!70!black]
         (ext) {};
  \end{scope}
      \node[paneltitle, text=orange!75!black, anchor=south west]
                        at ([xshift=6pt]ext.north west)
                        {Optional extensions};

\end{tikzpicture}
\caption{Per-drone control cascade and extension injection points.
Top framed panel: baseline continuous-time controller with outer-loop
PD and anti-swing slot shift, a single-step QP on the tilt/thrust
envelope, and inner attitude PD. The red local path is the key
feed-forward identity $T_i^{\mathrm{ff}} = T_i$, which maps the measured
rope tension directly into altitude thrust feed-forward and supplies
the bounded-jump mechanism used in \cref{thm:hybrid-stability}.
Bottom framed panel: optional extensions. The $L_1$ layer injects
$u_{\mathrm{ad}}$ between the outer PD and the QP, the MPC module
replaces the QP, and the reshape supervisor overrides the formation
slot $\boldsymbol{\delta}_i$ after a single
$\lceil\log_2 N\rceil$-bit broadcast.}
\label{fig:controller-cascade}
\end{figure*}
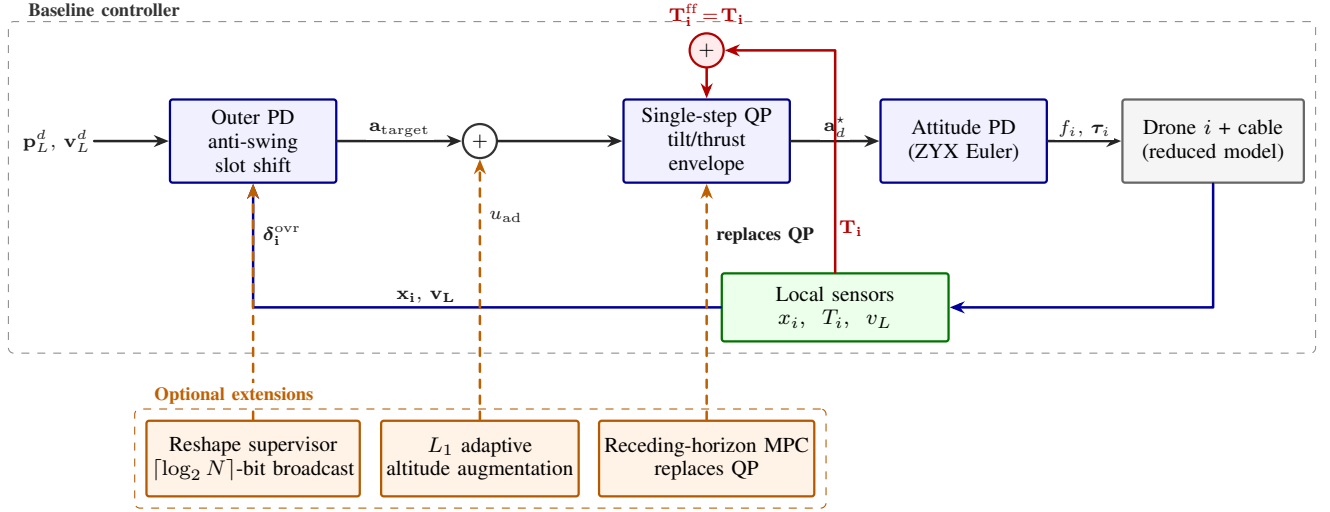

%% file: Section_IV_Stability_Analysis.tex
The baseline controller of \cref{sec:proposed} achieves fault
tolerance through a single structural choice: setting the altitude
thrust feed-forward equal to the drone's own measured rope tension,
$T_i^{\mathrm{ff}} = T_i$. We establish that this identity renders
the closed loop \emph{hybrid practical input-to-state stable}
(hybrid practical-ISS) under the cable-severance fault model of
\cref{sec:reduction-and-reference}. When a peer cable severs,
$T_i$ rises immediately, and the feed-forward converts that step
into a proportional thrust increment at the plant's native rate,
without any detection, annunciation, or reconfiguration interval.

The lemmas of \cref{sec:five-lemmas} verify the hybrid-systems hypotheses of~\cite{GoebelSanfeliceTeel2012} on $\Omega_\tau^{\mathrm{dwell}}$. The decay rates $\lambda,\rho$ are explicit from the Lyapunov matrices $P_v,P_{xy}$ in \cref{lem:altitude-iss}; the jump bound $\chi(\Delta_f)$ is structurally explicit via \eqref{eq:chi-explicit}, with the tension-weight $\kappa_V$ left as a calibration parameter; this does not affect the stated $\rho<1$ contraction condition.

% ----------------------------------------------------------
\subsection{Setup, Notation, and Analytical Hypotheses}
\label{sec:stability-setup}
% ----------------------------------------------------------

The analysis operates on the reduced-order model of
\cref{sec:reduction-and-reference}. Let
$e_{p,z,i}(t) = p_{\mathrm{slot},i,z}(t) - p_{i,z}(t)$ be
drone $i$'s altitude slot-tracking error and
$\vect{e}_p(t) \in \R^{3N_s}$ the stacked position error over
$\mathcal{S}(t)$ with $N_s = |\mathcal{S}(t)|$. Consecutive faults are separated by at least one
payload-pendulum period,
\begin{equation}
  \tau_{d,k} \;\ge\; \tau_{\mathrm{pend}}
    \;\triangleq\; 2\pi\sqrt{L/g}.
  \label{eq:dwell-condition}
\end{equation}
At $L = \SI{1.25}{\metre}$, $\tau_{\mathrm{pend}} =
\SI{2.24}{\second}$ and the pendulum natural frequency is
$\omega_p \triangleq \sqrt{g/L} = \SI{2.80}{\radian\per\second}$.

The analysis rests on three hypotheses, checked empirically in \cref{tab:stability-evidence,tab:domain-audit}: \textbf{H1} (slack-excursion budget, $\tau_{\mathrm{slack,max}}\le\SI{40}{\milli\second}$ and $\eta_{\max}\le 2.5\%$); \textbf{H2} (spaced-fault dwell, $\tau_{d,k}\ge\tau_{\mathrm{pend}}$ \eqref{eq:dwell-condition}); \textbf{H3} (Doctrine~D1, single-active-regime QP whose active set $\mathcal{A}(t)$ is locally constant inside each inter-fault window). The H1 gate bounds are set by bead-chain physics: $\SI{40}{\milli\second}$ is one quarter of the Tikhonov frozen window $\tau_{\mathrm{pend}}/(2\pi)$, and $2.5\%$ is a Rayleigh-turbulence budget. \Cref{thm:reduction} is independently validated against the bead-chain truth model in \cref{sec:VI-reduction}, providing an \emph{a priori} check that the reduction error is absorbed into $\gamma_\eta$ of \cref{thm:hybrid-stability} without appeal to the simulated missions; H1 and H3 are checked on the same mission family, and the corresponding margins are reported in \cref{tab:domain-audit}.

% ----------------------------------------------------------
\subsection{Lemmas Underpinning the Hybrid Practical-ISS Theorem}
\label{sec:five-lemmas}
% ----------------------------------------------------------

Four lemmas plus the anti-swing invariance remark (\cref{lem:antiswing-invariant}) underpin \cref{thm:hybrid-stability}: tension cancellation on the altitude channel (\cref{lem:tension-cancel}); ISS decay for altitude and horizontal channels via Lyapunov matrices $P_v,P_{xy}$ (\cref{lem:altitude-iss,lem:horizontal-iss}); a bounded Lyapunov jump $\chi(\Delta_f)$ at each fault (\cref{lem:jump-continuity}); and the dwell-cycle contraction $\rho\in(0,1)$ composing jump and decay (\cref{lem:dwell-decay}). Pre- and post-fault tension coordinates are compared on the surviving set by restricting the pre-fault vector to the remaining cables.

% ---- Lemma 2 (unchanged) ----------------------------------
\begin{lemma}[Tension cancellation]
\label{lem:tension-cancel}
Assume H1 and the cable-tilt domain condition
$|\theta_i(t)| \le \theta_{c,\max}$ for known
$\theta_{c,\max} \in (0, \pi/2)$. Then
$T_i^{\mathrm{ff}} = T_i$ cancels the dominant projection term
$T_i\cos\theta_i$ from the drone altitude acceleration,
leaving a residual factor $(1 - \cos\theta_i)$ bounded by
$1 - \cos\theta_{c,\max}$. At $\theta_{c,\max} =
\SI{0.75}{\radian}$ this is $0.27$, and the residual enters the altitude-error equation as a bounded matched disturbance
handled by \cref{lem:altitude-iss}. The bound holds uniformly
in $N_s \ge 2$.
\end{lemma}

\begin{proof}
Substituting $f_i = m_i(g + a_{d,z}^\star) + T_i^{\mathrm{ff}}$
into $m_i \ddot{p}_{i,z} = f_i - m_i g - T_i\cos\theta_i +
F_{w,i,z}$ and absorbing the wind term yields
\begin{equation}
  m_i \ddot{p}_{i,z}
    = m_i a_{d,z}^\star + T_i(1 - \cos\theta_i).
  \label{eq:alt-dynamics-residual}
\end{equation}
Since $|\theta_i| \le \theta_{c,\max}$ uniformly (enforced by the QP tilt ceiling), the residual $T_i(1-\cos\theta_i) \le T_i(1 - \cos\theta_{c,\max})$ regardless of $N_s$, and is absorbed by $\gamma_q$ of \cref{lem:altitude-iss}.
\end{proof}

% ---- Lemmas 3--4: ISS decay (altitude + horizontal) — merged
\begin{lemma}[ISS decay via Lyapunov matrices]
\label{lem:altitude-iss}
\label{lem:horizontal-iss}
Under H1--H3, both the altitude error $e_{p,z,i}$ and horizontal
error $e_{p,xy,i}$ satisfy exponential ISS between fault events.
Let $A_m\!=\!\begin{bmatrix}0&1\\-K_p^z&-K_d^z\end{bmatrix}$ and
$A_{m,xy}\!=\!\begin{bmatrix}0&1\\-K_p^{xy}&-K_d^{xy}\end{bmatrix}$
be the closed-loop matrices for gains $(K_p^z,K_d^z)\!=\!(100,24)$
and $(K_p^{xy},K_d^{xy})\!=\!(30,15)$.
The unique solutions $P_v \succ 0$ and $P_{xy} \succ 0$ of
$A_m^\top P_v + P_v A_m = -I$ and
$A_{m,xy}^\top P_{xy} + P_{xy} A_{m,xy} = -I$ are
\begin{equation}
  P_v = \begin{bmatrix} 2.224 & 0.005 \\ 0.005 & 0.021 \end{bmatrix},
  \qquad
  P_{xy} = \begin{bmatrix} 1.283 & 0.017 \\ 0.017 & 0.034 \end{bmatrix}.
  \label{eq:Pv-numeric}
\end{equation}
Setting $V_z = \tfrac{1}{2}[e_{p,z},\dot{e}_{p,z}]^\top P_v
[e_{p,z},\dot{e}_{p,z}]$ and analogously $V_{xy}$, there exist
$\gamma_w,\gamma_q,\gamma_w^{xy},\gamma_q^{xy} \in \mathcal{K}$
such that, between fault events,
\begin{align}
  \dot V_z &\;\le\; -\lambda_z\,V_z
    + \gamma_w(\|w\|_\infty) + \gamma_q(\|q - q^\star\|_\infty),
  \label{eq:lyapunov-decay-alt}\\
  \dot V_{xy} &\;\le\; -\lambda_{xy}\,V_{xy}
    + \gamma_w^{xy}(\|w\|_\infty) + \gamma_q^{xy}(\|q - q^\star\|_\infty),
  \label{eq:lyapunov-decay-xy}
\end{align}
with $\lambda_z = \alpha_z = |\mathrm{Re}(s_1)| =
\SI{5.37}{\second^{-1}}$ (dominant altitude pole
$s_1 = -12+\sqrt{44}$) and $\lambda_{xy} = \alpha_{xy} =
|\mathrm{Re}(s_1^{xy})| = \SI{2.38}{\second^{-1}}$ (dominant
horizontal pole $s_1^{xy} = (-15+\sqrt{105})/2$). Since
$\alpha_{xy} < \alpha_z$, the horizontal channel is the
bottleneck: $\alpha_{\min} = \alpha_{xy} =
\SI{2.38}{\second^{-1}}$, which drives the dwell-cycle
contraction of \cref{lem:dwell-decay}. Both rates are
independent of $N_s$.
\end{lemma}

\begin{proof}
After tension cancellation (\cref{lem:tension-cancel}), each scalar error satisfies a Hurwitz closed-loop ODE; for altitude:
\begin{equation}
  \ddot{e}_{p,z,i} + K_d^z \dot{e}_{p,z,i} + K_p^z e_{p,z,i}
    = \ddot{p}_{\mathrm{slot},z} + r_i(t).
  \label{eq:alt-error-ode}
\end{equation}
The unique $P \succ 0$ solving $A^\top P + PA = -I$ is computed algebraically, yielding the matrices in \eqref{eq:Pv-numeric}; differentiating $V = \tfrac{1}{2}\zeta^\top P\zeta$ along trajectories and applying Young's inequality gives \eqref{eq:lyapunov-decay-alt}--\eqref{eq:lyapunov-decay-xy} with rates $|\mathrm{Re}(s_1)|$. Independence from $N_s$ follows because neither $A_m$ nor $A_{m,xy}$ depends on $|\mathcal{S}|$.
\end{proof}

% ---- Remark: pendulum damping invariance (formerly Lemma 4) ---
\begin{remark}[Pendulum damping invariance]
\label{lem:antiswing-invariant}
The anti-swing damping coefficient is $N_s$-independent: the
restoring force $N_s T_i^{\mathrm{ss}} k_{\mathrm{swing}}
v_{L,\perp}/L$ with $T_i^{\mathrm{ss}} = m_L g/N_s$ satisfies
$N_s T_i^{\mathrm{ss}} = m_L g$, leaving $B_{\mathrm{swing}} =
m_L g\,k_{\mathrm{swing}}/L = \SI{62.8}{\newton\second\per\metre}$
and damping ratio $\zeta = B_{\mathrm{swing}}/(2m_L\omega_p)
= 1.12 > 1$ (overdamped) for any $N_s \ge 2$.
\end{remark}

% ---- Lemma 5 (jump — expanded proof, explicit chi) --------
\begin{lemma}[Bounded fault jump with explicit $\chi$]
\label{lem:jump-continuity}
At each fault event $t_f$, the Lyapunov function undergoes a
bounded jump,
\begin{equation}
  V(\vect\xi(t_f^+))
    \;\le\;
    V(\vect\xi(t_f^-))
    \;+\; \chi(\Delta_f),
  \qquad
  \chi \in \mathcal{K},
  \label{eq:jump-bound}
\end{equation}
where $\vect\xi = (\vect e_p, \vect e_v, \vect\varepsilon_T)$, the pre-fault tension coordinates are restricted to the
post-fault survivor set through the map
$\Pi_{\mathcal{S}^+\leftarrow\mathcal{S}^-}$, and
$\Delta_f = \|\vect T^{\mathrm{qs}}(t_f^+) -
\Pi_{\mathcal{S}^+\leftarrow\mathcal{S}^-}
\vect T^{\mathrm{qs}}(t_f^-)\|$ is the quasi-static tension
redistribution over the surviving cables induced by the
severance. On the admissible domain of H1--H3, the jump
function admits the bound
\begin{equation}
  \chi(\Delta_f)
    \;\le\;
    \bar c_f\bigl(\Delta_f + \Delta_f^2\bigr),
  \label{eq:chi-explicit}
\end{equation}
for some $\bar c_f > 0$ depending only on the admissible
domain and the Lyapunov weighting of the tension channel.
Under symmetric equal sharing at hover, the linear term
vanishes and \eqref{eq:chi-explicit} reduces to
$\chi(\Delta_f) \le \kappa_V\Delta_f^2/(2N_s)$.
\end{lemma}

\begin{proof}[Sketch]
The Kelvin--Voigt cable model keeps $(\vect p_i,\vect v_i,p_L,v_L)$ continuous across $t_f$ since no impulsive force is applied; on the severed rope, tension drops to zero within $\tau_{\mathrm{rope}}\approx\SI{6.3}{\milli\second}$ \eqref{eq:tau-rope}. Decompose $V=V_{\mathrm{track}}+\kappa_V V_{\mathrm{tension}}$. By kinematic continuity and pre-loaded slot reference, $V_{\mathrm{track}}(t_f^+)\le V_{\mathrm{track}}(t_f^-)$. The survivor tensions $T_i=k_s(\ell_i-L)^+$ are continuous functionals of the (continuous) kinematic state, so $\vect\varepsilon_T(t_f^+)=\Pi_{\mathcal{S}^+\leftarrow\mathcal{S}^-}\vect\varepsilon_T(t_f^-)-\vect b_f$ with $\|\vect b_f\|=\Delta_f$. Quadratic expansion of $V_{\mathrm{tension}}$ about $\Pi\vect\varepsilon_T(t_f^-)$ on the compact admissible domain yields $V_{\mathrm{tension}}(t_f^+)-V_{\mathrm{tension}}(t_f^-)\le c_{f,1}\Delta_f+c_{f,2}\Delta_f^2$. Absorbing $\kappa_V c_{f,j}$ into $\bar c_f$ gives \eqref{eq:chi-explicit}. Under the canonical symmetric hover ($\vect\varepsilon_T(t_f^-)=0$) the linear term vanishes and survivor-count algebra reduces \eqref{eq:chi-explicit} to $\chi(\Delta_f)\le\kappa_V\Delta_f^2/(2N_s)$.
\end{proof}

% ---- Lemma 6 (dwell-decay, expanded remark) ---------------
\begin{lemma}[Dwell-cycle contraction]
\label{lem:dwell-decay}
Let $t_k^\star$ and $t_{k+1}^\star$ be consecutive fault
events separated by dwell-time $\tau_{d,k} \ge
\tau_{\mathrm{pend}}$. Under the bounded fault jump of
\cref{lem:jump-continuity} and the ISS decays of
\cref{lem:altitude-iss,lem:horizontal-iss}, the Lyapunov
function satisfies a per-fault-cycle contraction
\begin{align}
  V(\vect\xi(t_{k+1}^{\star,-}))
    &\;\le\;
    \rho\,V(\vect\xi(t_k^{\star,-})) \;+\; c,
  \\
  \rho
    &\;\triangleq\;
    \exp\!\left(-\alpha_{\min}\,\tau_{\mathrm{pend}}\right),
  \quad \alpha_{\min} = \alpha_{xy} = \SI{2.38}{\second^{-1}},\nonumber
\end{align}
where
\begin{align}
\label{eq:c-def}
  c \;&=\; \rho\,\chi(\Delta_f) \;+\;
          \frac{c_0}{\alpha_{\min}},\\
  c_0 &= \sup_{t}\!
    \bigl[\gamma_w(\|w(t)\|_\infty)
          + \gamma_q(\|q - q^\star\|_\infty)
          + \gamma_\eta(\eta_{\max})\bigr].\nonumber
\end{align}
At canonical parameters, $\rho \approx
\exp(-2.38 \times 2.24) \approx 0.005$.
\end{lemma}

\begin{proof}
Applying \cref{lem:jump-continuity} at $t_k^\star$:
\begin{equation}
  V(t_k^{\star,+}) \;\le\; V(t_k^{\star,-}) + \chi(\Delta_f).
  \label{eq:jump-step}
\end{equation}
Between $t_k^{\star,+}$ and $t_{k+1}^{\star,-}$, the flow
satisfies $\dot V \le -\alpha_{\min} V + c_0$ (combining
\cref{lem:altitude-iss,lem:horizontal-iss} through the
composite Lyapunov function $V = V_z + V_{xy} + \kappa_V
V_{\mathrm{tension}}$, using $\alpha_{\min} =
\min(\alpha_z, \alpha_{xy}) = \alpha_{xy}$). Integrating
over an interval of length $\tau \ge \tau_{\mathrm{pend}}$:
\begin{align}
  V(t_{k+1}^{\star,-})
    \;\le\;
    e^{-\alpha_{\min}\tau}&\bigl(V(t_k^{\star,-})
      + \chi(\Delta_f)\bigr)\\
    &\;+\; \frac{c_0}{\alpha_{\min}}(1 - e^{-\alpha_{\min}\tau}).\nonumber
\end{align}
Setting $\rho = e^{-\alpha_{\min}\tau_{\mathrm{pend}}}$ and
upper-bounding $\tau$ by $\tau_{\mathrm{pend}}$ (the minimum
dwell, giving the worst-case exponential decay) yields
\eqref{eq:dwell-contraction}--\eqref{eq:c-def}.
Since $\alpha_{\min} = \alpha_{xy} > 0$ and
$\tau_{\mathrm{pend}} > 0$, we have $\rho \in (0,1)$.
\end{proof}

\begin{remark}[Analytical $\rho$ vs.\ empirical $\hat\rho$]
\label{rem:rho-gap}
\Cref{lem:dwell-decay} gives $\rho_{\mathrm{ana}}=\exp(-\alpha_{xy}\tau_{\mathrm{pend}})\approx 0.005$, while measured V4/V5 Lyapunov-proxy ratios are $\hat\rho_{\mathrm{V4}}=0.20,\hat\rho_{\mathrm{V5}}=0.045$. The gap is the additive disturbance floor $c>0$:
\begin{equation}
  \hat\rho_k \triangleq
    \frac{V(t_{k+1}^{\star,-})}{V(t_k^{\star,-})}
    \le \rho_{\mathrm{ana}} + \frac{c}{V(t_k^{\star,-})},
  \label{eq:rho-hat-decomp}
\end{equation}
tight only in the zero-disturbance limit. Under Dryden $\SI{4}{\metre\per\second}$ wind, $c/V(t_k^{\star,-})\approx 0.195$ (V4) and $\approx 0.040$ (V5); the V4$\to$V5 improvement reflects further pre-fault decay over the longer dwell.
\end{remark}

% ----------------------------------------------------------
\subsection{Main Theorem, Robustness, and Multi-Fault Feasibility}
\label{sec:main-theorem-and-robustness}
% ----------------------------------------------------------

We now compose the five lemmas into the hybrid practical-ISS
theorem and state corollaries on robustness, the steady-state
tracking bound, and multi-fault feasibility.

\begin{theorem}[Hybrid practical-ISS with explicit recovery
  envelope under $F$ sequential agent failures]
\label{thm:hybrid-stability}
Under hypotheses H1--H3 (\cref{sec:stability-setup}), the
cable-tilt domain condition of \cref{lem:tension-cancel}, any
fault schedule $\sigma = \{(t_k^\star, i_k^\star)\}_{k=1}^{F}
\in \mathcal{F}$ of $F$ unannounced cable severances satisfying
$F \le N - 2$ and the per-fault actuator-margin condition
$m_L g/(N{-}k) \le \kappa_{\mathrm{act}} f_{\max}$ for all
$k \in \{1,\dots,F\}$, and admissible disturbances
$d \in \mathcal{D}$, the stacked tracking-and-tension-error
state $\vect\xi = (\vect e_p, \vect e_v, \vect\varepsilon_T)$
satisfies
\begin{align}
\label{eq:iss-envelope-explicit}
  \|\vect\xi(t)\|
    \;\le\;
      \beta\bigl(\|\vect\xi(0)\|,\,&t\bigr)
      + \gamma_w\bigl(\|w\|_\infty\bigr)
      + \gamma_q\bigl(\|q - q^\star\|_\infty\bigr)
      \notag \\
    &+ \gamma_\eta\bigl(\eta\bigr)
      + \sum_{k=1}^{F}
           \Gamma_f\,e^{-\lambda(t - t_k^\star)}\,
           \mathbf{1}_{t \ge t_k^\star},\nonumber\\
    &t \in [t_{\mathrm{pickup}},\, T],
\end{align}
where $\beta \in \mathcal{KL}$, $\gamma_w, \gamma_q,
\gamma_\eta \in \mathcal{K}$, $\lambda \triangleq
\alpha_{\min}/2 = \alpha_{xy}/2 = \SI{1.19}{\second^{-1}}$
is the norm-level recovery rate induced by the slower of the
altitude and horizontal closed-loop spectral rates
(\cref{lem:altitude-iss,lem:horizontal-iss}), $\eta =
\eta(\vect\xi[\cdot])$ is the realized slack duty cycle bounded
by $\eta_{\max}$ via H1, and $\Gamma_f > 0$ is the
fault-overshoot coefficient set by the redistribution geometry.
Equivalently, $V(\vect\xi(t_{k+1}^{\star,-})) \le
\rho\,V(\vect\xi(t_k^{\star,-})) + c$ for explicit $c > 0$
per \cref{lem:dwell-decay}. Each surviving cable's tension
obeys the practical envelope
\begin{align}
\label{eq:tension-envelope-thm}
  |T_i(t) - T_i^{\mathrm{qs}}(t)|
    \;\le\;
    &T_{\mathrm{overshoot}}\,
    e^{-\lambda(t - t^\star_{\mathrm{last}})}
    + T_{\mathrm{floor}},\\
    &i \in \mathcal{S}(t),\nonumber
\end{align}
where $T_{\mathrm{overshoot}}$ is bounded by the corresponding
tension-error radius extracted from \eqref{eq:chi-explicit},
and $T_{\mathrm{floor}}$ is the steady-state floor induced by
$\gamma_w$, $\gamma_q$, and $\gamma_\eta$.
\end{theorem}

\begin{proof}[Sketch]
\textit{Inter-fault ISS.} On each inter-fault interval $[t_k^{\star,+},t_{k+1}^{\star,-}]$, \cref{lem:tension-cancel} reduces the altitude dynamics to \eqref{eq:alt-error-ode} with bounded residual on $\Omega_\tau^{\mathrm{dwell}}$. Using the composite $V=\lambda_{\min}(P_v)^{-1}V_z+\lambda_{\min}(P_{xy})^{-1}V_{xy}+\kappa_V V_{\mathrm{tension}}$, \cref{lem:altitude-iss,lem:horizontal-iss} give $\dot V\le-\alpha_{\min}V+\gamma_w+\gamma_q+\gamma_\eta$ with $\alpha_{\min}=\alpha_{xy}$ and $\gamma_\eta$ absorbing the reduction residual of \cref{thm:reduction}. The comparison lemma yields $\beta\in\mathcal{KL}$ and $\gamma_w,\gamma_q,\gamma_\eta\in\mathcal{K}$. Quadratic equivalence $V\sim\|\xi\|^2$ converts the Lyapunov decay rate to the norm rate $\lambda=\alpha_{\min}/2$ (see \cref{rem:jump-vs-rate}).

\textit{Per-fault jump.} \Cref{lem:jump-continuity} gives $V(t_k^{\star,+})\le V(t_k^{\star,-})+\chi(\Delta_{f,k})$, and \cref{lem:dwell-decay} composes jump and decay into $V(t_{k+1}^{\star,-})\le\rho V(t_k^{\star,-})+c$. The exponential damping of the jump produces the $\Gamma_f e^{-\lambda(t-t_k^\star)}$ term in \eqref{eq:iss-envelope-explicit}.

\textit{Induction over $F\le N-2$ faults.} Geometric summation gives $V(t_k^{\star,-})\le\rho^{k-1}V(t_1^{\star,-})+c(1-\rho^{k-1})/(1-\rho)$, bounded as long as the actuator-margin condition $m_L g/(N-k)\le\kappa_{\mathrm{act}}f_{\max}$ holds for each $k$. Composing through the hybrid-system framework~\cite[Thm.~1]{GoebelSanfeliceTeel2012} yields \eqref{eq:iss-envelope-explicit}; the tension envelope \eqref{eq:tension-envelope-thm} follows from $V_{\mathrm{tension}}\sim\|\vect\varepsilon_T\|^2$ and $\varepsilon_{T,i}=T_i-T_i^{\mathrm{qs}}$.
\end{proof}

\begin{remark}[Jump-bound vs.\ contraction-rate mechanism]
\label{rem:jump-vs-rate}
The ablation (\cref{sec:VI-D}) shows that disabling $T_i^{\mathrm{ff}}=T_i$ increases peak sag by $3.6$--$4.0\times$ and produces larger post-fault Lyapunov jumps (\cref{fig:fault-zoom}), supporting the jump-bound role of the identity versus the contraction-rate role of $\alpha_{xy}$ within the tested architecture.
\end{remark}

% ---- Corollaries (unchanged except minor notation) --------
\begin{corollary}[Pre-fault steady-state tracking bound]
\label{cor:steady-state}
Prior to any fault event, the payload tracking error is
dominated by the quasi-static pendulum offset and the outer-loop
PD residual:
\begin{equation}
  \bar{e}_{p,xy} \le \frac{L A_{xy,\max}}{g}
    \left(1 - \frac{1}{2\zeta^2}\right)
    + \frac{A_{xy,\max}}{\kappa_{\mathrm{QP}} K_p^{xy}}
    + O(\delta),
  \label{eq:ss-bound}
\end{equation}
where $A_{xy,\max}$ is the maximum reference acceleration,
$\zeta = 1.12$ from \cref{lem:antiswing-invariant}, and
$\kappa_{\mathrm{QP}} = w_t/(w_t + w_e) \approx 0.98$. At
canonical $A_{xy,\max} \approx \SI{2.47}{\metre\per\second\squared}$,
\eqref{eq:ss-bound} gives $\bar{e}_{p,xy} \approx
\SI{0.27}{\metre}$, consistent with V1 horizontal RMSE
$\SI{0.306}{\metre}$ (gap from centripetal acceleration on
curved segments, excluded by the quasi-static derivation).
\end{corollary}

\begin{corollary}[Robustness to mass, stiffness, drag, and wind]
\label{cor:robustness}
The envelope \eqref{eq:iss-envelope-explicit} is robust along four axes: (i)~mass mismatch: $T_i^{\mathrm{ff}} = T_i$ uses measured tension, making \cref{lem:tension-cancel,lem:altitude-iss} mass-independent to leading order; (ii)~stiffness and drag mismatches: enter through $O(\delta + \eta_{\max})$, absorbed by $\gamma_\eta$ without degrading $\rho$ or $\alpha_{xy}$; (iii)~wind: enters through $\gamma_w$ with steady-state gain $1/K_p^{xy}$. Empirically: P2-B yields $<1.9\%$ RMSE variation over a $56\%$ mass range with FF active; V2 yields $\SI{0.312}{\metre}$ RMSE under Dryden wind, matching V1 (\cref{cor:steady-state}).
\end{corollary}

\paragraph{Multi-fault feasibility.}
\eqref{eq:iss-envelope-explicit} extends to $k$ sequential
faults by induction under two conditions: (a)~geometric
closure $k \le N - 2$, (b)~actuator margin
$m_L g/(N{-}k) \le \kappa_{\mathrm{act}} f_{\max}$ at every
step. Canonical $N = 5$, $m_L = \SI{10}{\kilo\gram}$ admits
$k = 2$ with margin; V4 and V5 (dual-fault) yield
3-D RMSE $0.324$--$\SI{0.328}{\metre}$, consistent with
\cref{thm:hybrid-stability} on the demonstrated mission family.

\paragraph{Empirical validation summary.}
\label{sec:stability-empirical-validation}
\Cref{tab:stability-evidence} maps each claim to its campaign
observation; full context is in \cref{sec:simulation}.

\begin{table*}[t]
  \centering
  \caption{Empirical support for hybrid practical-ISS claims.}
  \label{tab:stability-evidence}
  \small
  \renewcommand{\arraystretch}{1.05}
  \begin{tabular}{p{5.5cm}p{2.5cm}p{8cm}}
    \toprule
    Claim / Lemma & Campaign & Evidence \\
    \midrule
    Tension cancellation (\cref{lem:tension-cancel})
      & P2-A ablation
      & $3.6\times$--$4.0\times$ sag increase with FF
        disabled \\
    Altitude ISS decay (\cref{lem:altitude-iss})
      & V1 nominal
      & 3-D RMSE $= \SI{0.312}{\metre}$, stable;
        $\lambda_z = \SI{5.37}{\second^{-1}}$ from $P_v$
        in \eqref{eq:Pv-numeric} \\
    Horizontal ISS decay (\cref{lem:horizontal-iss})
      & V1--V5 series
      & $\alpha_{xy} = \SI{2.38}{\second^{-1}}$
        from $P_{xy}$ in \eqref{eq:Pv-numeric};
        $\rho = \exp(-\alpha_{xy}\tau_{\mathrm{pend}})
        \approx 0.005$ \\
    Anti-swing invariance (\cref{lem:antiswing-invariant})
      & V1--V5 series
      & Post-fault RMSE within $4\%$ of nominal
        across all variants \\
    Bounded fault jump (\cref{lem:jump-continuity})
      & V3, V4, V5 traces
      & $\hat\chi$ peak $1.1\!\times\!10^{-4}$ (V4),
        $1.3\!\times\!10^{-5}$ (V5); qualitatively consistent
        with the bounded redistribution estimate \eqref{eq:chi-explicit}; under equal-sharing, the canonical coefficient reduces to $\kappa_V\Delta_f^2/(2N_s)$, with quantitative calibration of $\kappa_V$ left for future work \\
    Dwell-cycle contraction $\rho < 1$
        (\cref{lem:dwell-decay})
      & V4, V5 dual-fault
      & $\hat\rho_{\mathrm{V4}} = 0.20$,
        $\hat\rho_{\mathrm{V5}} = 0.045$;
        both satisfy \eqref{eq:rho-hat-decomp}
        with $c/V(t_k^{\star,-})$ set by the Dryden
        disturbance floor (\cref{rem:rho-gap}) \\
    Mass robustness (\cref{cor:robustness})
      & P2-B sweep
      & $1.9\%$ RMSE range over $56\%$ mass variation \\
    Wind robustness (\cref{cor:robustness})
      & V2 wind
      & V2 3-D RMSE $= \SI{0.312}{\metre} \approx$ V1 \\
    \bottomrule
  \end{tabular}
\end{table*}

%% file: Section_V_Parameter_Adaptation_Tension_Constraints_and_Post-Fault_Formation_Reshape.tex
Three extension layers augment the baseline cascade of \cref{sec:proposed} and address residual concerns that the identity $T_i^{\mathrm{ff}}=T_i$ does not target: (i) parameter mismatch in $(m_L, k_{\mathrm{eff}})$, (ii) horizon-aware tension-ceiling enforcement, and (iii) post-fault hover-equilibrium asymmetry. Each layer enters the closed loop through a disjoint channel. The $L_1$ altitude augmentation and the MPC tension-ceiling layer preserve the hypotheses H1--H3 of \cref{thm:hybrid-stability} without peer communication for control; on the simulated trajectory family they remain inactive unless their corresponding regime is encountered, with the MPC layer specifically inactive in NF2. The reshape supervisor likewise remains inactive in NF3, but it is governed by the separate result of \cref{thm:reshape} and, unlike the other two layers, uses a local detection latch and a $\lceil\log_2 N\rceil$-bit broadcast. This section states the role of each layer and the architectural invariant that keeps the main result \eqref{eq:iss-envelope-explicit} valid when any combination is enabled. The adaptation laws, MPC formulation, reshape geometry, and preservation arguments are provided in the supplementary material.

\subsection{$L_1$-Adaptive Altitude Augmentation}
\label{sec:l1-augmentation}

The identity $T_i^{\mathrm{ff}}=T_i$ does not cancel matched disturbances entering through parameter deviations $(m_L, k_{\mathrm{eff}})$. The $L_1$ layer~\cite{Hovakimyan2010book,CaoHovakimyan2008} estimates the residual online and injects an additive altitude correction $u_{\mathrm{ad}}(t)$ between the outer PD and the QP at the injection point of \cref{fig:controller-cascade}. The standard $L_1$ architecture (Hurwitz reference model with $A_m,P_v$ from \cref{lem:altitude-iss}, state predictor, projected gradient on $\hat\delta_m \in [\delta_{m,\min},\delta_{m,\max}]$, low-pass filter at $\omega_c=\SI{25}{\radian\per\second}$) is used; the discrete update runs at $T_s=\SI{2e-4}{\second}$ with adaptation gain $\Gamma$. Contractivity of the scalar Euler-discretized update yields the design rule below; the full derivation is in the supplementary material.

\begin{proposition}[Scalar $L_1$ admissible-window bound]
\label{prop:gamma-star}
The Euler-discretized $L_1$ update at step $T_s$ admits two necessary conditions on the adaptation gain $\Gamma$:
\begin{equation}
  \Gamma_{\min} \;\triangleq\; \frac{\omega_c}{p_{22}}
   \;<\; \Gamma \;<\;
   \frac{2}{T_s\,p_{22}} \;\triangleq\; \Gamma^\star,
   \label{eq:gamma-window}
\end{equation}
expressing bandwidth coverage and discretization stability respectively. For $T_s=\SI{2e-4}{\second}$, $p_{22}=0.0210$, $\omega_c=\SI{25}{\radian\per\second}$, this gives $\Gamma_{\min}\approx 1{,}190 < \Gamma=2{,}000 < \Gamma^\star\approx 4.75\times 10^5$, placing the canonical operating point $1.68\times\Gamma_{\min}$ above the lower bound with a $238\times$ discretization safety margin.
\end{proposition}

\subsection{Receding-Horizon MPC with a Penalized Tension-Ceiling Constraint}
\label{sec:mpc-layer}

The single-step QP \eqref{eq:qp} cannot foresee a tension spike ahead. When safe working load dominates, we replace it with an $N_p$-step receding-horizon MPC carrying a soft penalized tension-ceiling constraint ($\vect s\ge\vect 0$, weight $w_s=10^4$). On each tick, the MPC propagates a ZOH discretization of the slot-tracking error and a chord-linearized tension prediction, then solves a single warm-started OSQP~\cite{stellato2020osqp} program; only the first input is applied. Slack penalization preserves unconditional feasibility (\cref{cor:mpc-recfeas}); under D1 and Lipschitz-in-parameter right-hand sides, slack activation is $O(\bar\varepsilon_q)$ in the parameter-mismatch radius $\bar\varepsilon_q=\|q-q^\star\|_\infty$ (\cref{prop:mpc-slack}). The MPC reduces to the baseline QP whenever the predicted tension remains below the ceiling. Canonical settings: $N_p=5$ ($10$ at evaluation), $\Delta t_{\mathrm{MPC}}=\SI{0.01}{\second}$, $T_{\max}=\SI{100}{\newton}$, $k_{\mathrm{eff}}=\SI{2778}{\newton\per\metre}$.

\begin{corollary}[Unconditional feasibility]
\label{cor:mpc-recfeas}
The horizon QP is feasible at every tick: $\mathcal{U}=\vect 0$ with $\vect s$ chosen entry-wise as the positive part of the deficit is always admissible.
\end{corollary}

\begin{proposition}[Slack-activation regime]
\label{prop:mpc-slack}
Let $q=(m_L,k_{\mathrm{eff}})$, $\|q-q^\star\|_\infty\le\bar\varepsilon_q$, and let the constraint right-hand side be $L_q$-Lipschitz in $q$. Under D1 and an initial tick where the unconstrained QP satisfies the ceiling, $\|\vect s[k]\|_\infty\le L_q\bar\varepsilon_q + O(\bar\varepsilon_q^2)$.
\end{proposition}

\subsection{Fault-Triggered Formation-Reshape Supervisor}
\label{sec:reshape-supervisor}

After a severance, the survivors retain their pre-fault angular slots, leaving the hover-equilibrium tension distribution asymmetric. This asymmetry is the binding load only when the trajectory is sufficiently low-bandwidth that the dynamic centripetal component does not dominate (NF3, \cref{rem:reshape-null}). The reshape supervisor reassigns survivors to an equiangular configuration using a $ C^2$-smooth quintic smoothstep over $T_{\mathrm{trans}}=\SI{5}{\second}$. Unlike the baseline, this is an \emph{explicitly supervised} layer: each drone runs a local tension-latch detector ($T_{\mathrm{fault}}=\SI{0.5}{\newton}$, $\tau_{\mathrm{det}}=\SI{100}{\milli\second}$) and a per-fault broadcast of at most $\lceil\log_2 N\rceil$ bits distributes the fault index. The no-detection and no-communication statement of \cref{thm:hybrid-stability} therefore do not extend to the reshape layer.

\begin{theorem}[Globally tension-optimal reshape, $N=4$]
\label{thm:reshape}
For $N=4$ with nominal slots $\phi_i=\pi i/2$ and severed cable $i^\star$, under the symmetric-load-sharing model, the worst-case-tension-minimizing reassignment over the survivors places the diametrically opposite drone at its nominal angle and rotates the two adjacent drones by $\pi/6$ toward the gap. The resulting equiangular $120^\circ$ spacing minimizes $\max_i T_i$, with closed-form worst-case reduction $25.8\%$.
\end{theorem}

The proof, based on reflection symmetry and a piecewise-quadratic minimax argument in $\Delta\phi$, and the $V_4\to V_6$ chained application across $N=5\to 4\to 3$ are provided in the supplementary material.

\subsection{Baseline Theorem Preservation}
\label{sec:theorem-preservation}

Each extension modifies the closed loop through a disjoint channel and preserves H1--H3 of \cref{thm:hybrid-stability} on $\Omega = \Omega_\tau^{\mathrm{dwell}}\cap\Omega_{\mathrm{QP,active}}$. The $L_1$ layer adds $u_{\mathrm{ad}}$ to the altitude component of $\vect a_{\mathrm{target}}$ before the QP, leaving the active set and the identity unchanged; H1, H3 hold verbatim, H2 is controller-independent, and the $\gamma_q$-gain in \eqref{eq:iss-envelope-explicit} is reduced by the composite $L_1$ contraction. The MPC layer replaces \eqref{eq:qp} with the horizon QP of \cref{sec:mpc-layer}: \cref{cor:mpc-recfeas} guarantees feasibility, \cref{prop:mpc-slack} bounds the slack, and the closed loop reduces to the baseline whenever the ceiling is loose. The reshape supervisor injects a $C^2$-bounded slot perturbation via the override hook; H1 and H3 hold provided the smoothstep keeps the perturbation within the admissible-reference set, and H2 reduces to a latching rule that fires at most once per dwell period. \eqref{eq:iss-envelope-explicit} therefore applies under the stated combinations of enabled layers; cross-layer interactions are evaluated empirically in \cref{sec:VI-E,sec:VI-F}.

\paragraph{Out-of-scope conditions.} Three regimes lie outside the analysis domain: rapid multi-fault ($\tau_d<\tau_{\mathrm{pend}}$), long slack excursions (run $>\SI{40}{\milli\second}$ or duty $>2.5\%$), and QP-fallback, violating H2, H1, and H3 respectively. A second class lies outside the analysis domain: drone-side faults (rotor loss, brownout), soft cable failures (fraying, attachment slip), tension-sensor failures, and payload-attitude dynamics. These cases require additional models or sensing assumptions and are not addressed in the present campaign.

%% file: Section_VI_Simulation_Campaign_and_Results.tex
\begin{figure*}[!t]
  \centering
  \makebox[\textwidth][c]{%
    \begin{minipage}[c]{0.55\textwidth}
      \centering
      \includegraphics[width=\linewidth]{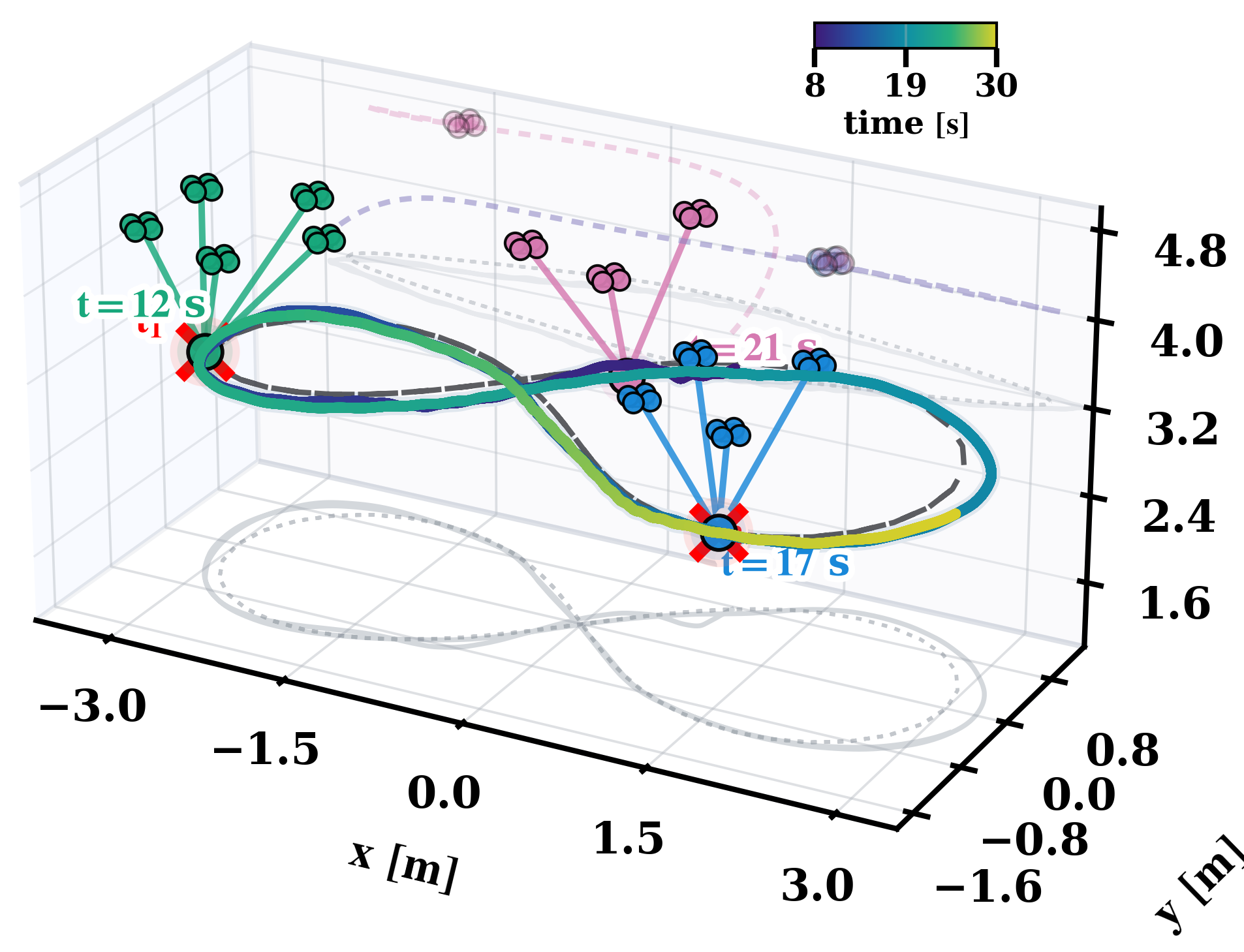}
    \end{minipage}\hspace{0.02\textwidth}%
    \begin{minipage}[c]{0.40\textwidth}
      \centering
      \includegraphics[width=\linewidth]{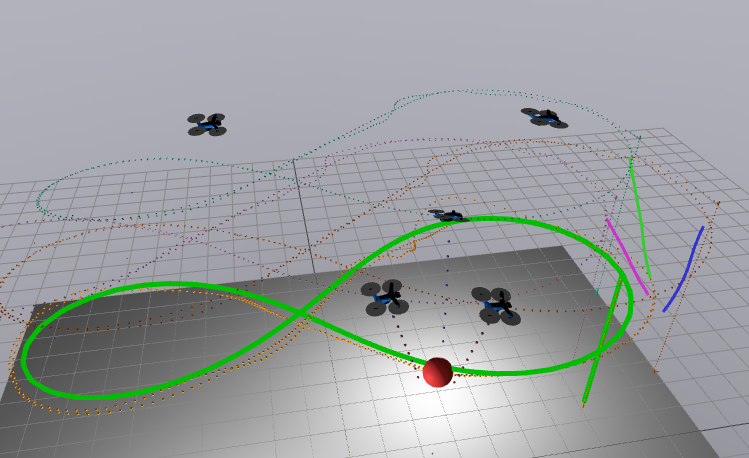}
    \end{minipage}%
  }
  \caption{V4 dual-fault scenario in analysis and in Drake. Left: time-coloured payload trajectory over $t \in [8,30]$\,\si{\second} with dashed lemniscate reference, fault events at $t_1{=}12$\,s (drone~0) and $t_2{=}17$\,s (drone~2), and formation snapshots at $t\in\{12,17,21\}$\,s showing fault-1, fault-2, and post-fault recovery geometries. Detached quadcopters are shown with released outbound traces, and grey shadows show top and side projections. Right: representative Drake simulation snapshot of the same five-drone payload-transport scenario with compliant ropes and the suspended load.}
  \label{fig:trajectory-3d-V4}
\end{figure*}

\begin{figure}[htbp]
  \centering
  \includegraphics[width=\columnwidth]{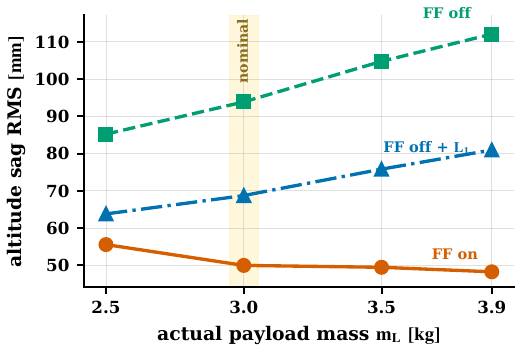}
  \caption{P2-B payload-mass-mismatch sweep. Altitude sag RMS versus $m_L \in \{2.5,3.0,3.5,3.9\}~\si{\kilo\gram}$ with $q^\star = \SI{3.0}{\kilo\gram}$ for FF-on, FF-off, and FF-off$+L_1$. The shaded band marks the nominal controller mass. FF-on stays nearly flat across the sweep, while FF-off$+L_1$ recovers roughly $40\,\%$ of the FF-off-minus-FF-on sag gap.}
  \label{fig:l1-mass-mismatch}
\end{figure}

\begin{figure}[t]
  \centering
  \includegraphics[width=\columnwidth]{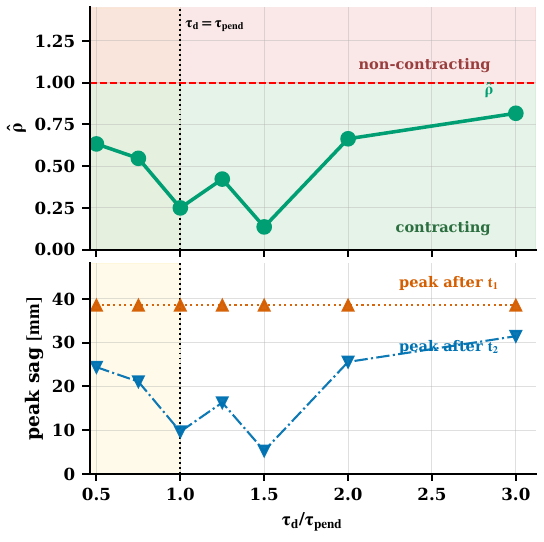}
  \caption{Dwell-cycle peak-excursion contraction $\hat\rho(\tau_d)$ versus normalised dwell $\tau_d/\tau_{\mathrm{pend}}$, V4 schedule, seed~42. Horizontal dashed: $\hat\rho=1$; vertical dotted: H2 boundary. $\hat\rho$ remains below unity across the full sweep, including sub-threshold dwell.}
  \label{fig:dwell-sweep}
\end{figure}

All simulations run in Drake~\cite{drake} with an RK3 continuous-time integrator at a max step of \SI{1}{\milli\second}. Drones and payload are rigid multibody systems; ropes are bead-chain Kelvin--Voigt spring-dampers whose lumped stiffness is justified by \cref{thm:reduction}. Wind is body-frame aerodynamic drag per drone and on the payload, using a Dryden low-altitude turbulence model. \Cref{tab:sim-params} lists the simulator and plant parameters.

\begin{table}[htbp]
  \caption{Simulator and Plant Parameters}
  \label{tab:sim-params}
  \centering
  \begin{tabular}{lrl}
    \toprule
    Parameter                        & Value     & Unit                       \\
    \midrule
    Number of drones $N$             & 5         & ---                        \\
    Drone mass $m_i$                 & 1.5       & \si{\kilo\gram}            \\
    Payload mass $m_L$               & 10        & \si{\kilo\gram}            \\
    Rope rest length $\ell_0$        & 1.25      & \si{\metre}                \\
    Rope stiffness $k$               & 25\,000   & \si{\newton\per\metre}     \\
    Formation ring radius            & 0.80      & \si{\metre}                \\
    Reference trajectory             & 3-D lemniscate & ---               \\
    \quad Horizontal amplitude       & 3.0       & \si{\metre}                \\
    \quad Period (nominal)           & 12        & \si{\second}               \\
    \quad Vertical amplitude         & 0.35      & \si{\metre}                \\
    Wind model                       & Dryden low-altitude & ---          \\
    \quad Mean speed (along $+x$)    & 4         & \si{\metre\per\second}     \\
    \quad $\sigma_u,\;\sigma_v$      & 0.8       & \si{\metre\per\second}     \\
    \quad $\sigma_w$                 & 0.4       & \si{\metre\per\second}     \\
    Turbulence seed                  & 42        & ---                        \\
    Simulation duration              & 30        & \si{\second}               \\
    Pre-fault cruise window          & $[8,\,12]$  & \si{\second}             \\
    Actuator thrust ceiling          & 150       & \si{\newton}               \\
    MPC tension ceiling (V6 only)    & 100       & \si{\newton}               \\
    \bottomrule
  \end{tabular}
\end{table}

The campaign structure (\cref{tab:campaign-matrix}) comprises V1--V6 as the pre-declared capability sequence, each adding one stressor, and P2-A--P2-D as targeted ablations. Faults are instantaneous: the rope tension is set to zero; survivors continue without notification.

\begin{table*}[htbp]
  \caption{Simulation campaign matrix. Capability-demonstration variants V1--V6 and targeted ablations P2-A through P2-D.}
  \label{tab:campaign-matrix}
  \centering
  \begin{tabular}{llp{5.5cm}cp{5cm}}
    \toprule
    Tag & Controller            & Fault Schedule                                                                          & Wind & Purpose \\
    \midrule
    \multicolumn{5}{l}{\textit{Capability Demonstration and Admissibility Check}} \\[\smallskipamount]
    V1 & Baseline QP           & None                                                                                    & No   & Nominal tracking floor; C2 bound check \\
    V2 & Baseline QP           & None                                                                                    & Yes  & Wind-rejection baseline \\
    V3 & Baseline QP           & Drone~0 at $t{=}\SI{12}{s}$                                                             & Yes  & Single-fault recovery \\
    V4 & Baseline QP           & Drone~0 at $t{=}\SI{12}{s}$; Drone~2 at $t{=}\SI{17}{s}$ \mbox{($\Delta t{=}\SI{5}{s}$)}  & Yes  & Dual fault; dwell margin \SI{2.76}{s} \\
    V5 & Baseline QP           & Drone~0 at $t{=}\SI{12}{s}$; Drone~2 at $t{=}\SI{22}{s}$ \mbox{($\Delta t{=}\SI{10}{s}$)} & Yes  & Dual fault; dwell margin \SI{7.76}{s} \\
    V6 & Full stack            & Same as V4                                                                              & Yes  & Full-stack integration; hardest scenario \\
       & (MPC$+L_1+$reshape)  &                                                                                         &      & \\
    \midrule
    \multicolumn{5}{l}{\textit{Targeted Ablations and Operating-Domain Probes}} \\[\smallskipamount]
    P2-A & Baseline, FF disabled & V3, V4, V5 schedules                                                                 & Yes  & Feed-forward mechanism evidence (C2) \\
    P2-B & Baseline / $+L_1$    & None                                                                                  & No   & Mass-mismatch robustness (C3-$L_1$) \\
    P2-C & MPC, ceiling sweep   & V4 schedule                                                                           & Yes  & Tension-ceiling sensitivity (C3-MPC) \\
    P2-D & Baseline + reshape   & V4 schedule, period $T_{\mathrm{ref}} \in \{8,10,12\}$~s                              & Yes  & Reshape benefit vs.\ trajectory period (C3-Reshape) \\
    \bottomrule
  \end{tabular}
\end{table*}

\paragraph{Admissibility-domain check (H1 and H3).} \Cref{thm:reduction} and \cref{thm:hybrid-stability} require three gates to hold on the demonstrated trajectory: H1a (no slack run $> \SI{40}{\milli\second}$), H1b (aggregate slack duty cycle $\le 2.5\%$), and H3 (QP active-set transition fraction $\le 1\%$). \Cref{tab:domain-audit} reports measurements for V1--V6: every variant passes with margin. Worst observed values are \SI{35.8}{\milli\second} (H1a; V2--V5, $\SI{4.2}{\milli\second}$ margin), $1.66\%$ (H1b; V6, $0.84$~pp margin), and $0.10\%$ (H3; V4--V6, $0.90$~pp margin). These measurements show domain membership for the demonstrated mission family but do not prove H1--H3 for all references, wind, or fault schedules; the check must be repeated for new missions. V6's lower H1a ($\SI{31.4}{\milli\second}$) is consistent with MPC damping high-frequency QP transitions.

\begin{table}[htbp]
  \caption{Domain gate check. Thresholds: H1a $\le \SI{40}{\milli\second}$; H1b $\le 2.5\,\%$; H3 $\le 1.0\,\%$.}
  \label{tab:domain-audit}
  \centering
  \small
  \renewcommand{\arraystretch}{1.05}
  \begin{tabular}{lcccccc}
    \toprule
    & \multicolumn{2}{c}{H1a (ms)} & \multicolumn{2}{c}{H1b (\%)} & \multicolumn{2}{c}{H3 (\%)} \\
    \cmidrule(lr){2-3}\cmidrule(lr){4-5}\cmidrule(lr){6-7}
    Variant & Value & Pass & Value & Pass & Value & Pass \\
    \midrule
    V1 & 21.2 & \checkmark & 1.34 & \checkmark & 0.07 & \checkmark \\
    V2 & 35.8 & \checkmark & 1.52 & \checkmark & 0.07 & \checkmark \\
    V3 & 35.8 & \checkmark & 1.52 & \checkmark & 0.09 & \checkmark \\
    V4 & 35.8 & \checkmark & 1.52 & \checkmark & 0.10 & \checkmark \\
    V5 & 35.8 & \checkmark & 1.52 & \checkmark & 0.10 & \checkmark \\
    V6 & 31.4 & \checkmark & 1.66 & \checkmark & 0.10 & \checkmark \\
    \midrule
    Gate & \multicolumn{2}{c}{$\le 40$} & \multicolumn{2}{c}{$\le 2.5$} & \multicolumn{2}{c}{$\le 1.0$} \\
    \bottomrule
  \end{tabular}
\end{table}

\paragraph{Statistical scope.} All campaigns use fixed Dryden seed~42 (reproducible $2\sigma$ exceedance within $[8,12]\,\si{\second}$); reported RMSE, sag, and ablation values are deterministic point estimates. CRN-paired replication with bootstrap CIs is follow-up work.

\paragraph{Scope.} V1--V6 begin slightly slack ($\SI{1.17}{\metre}$ vs.\ $\SI{1.25}{\metre}$ rest length); pickup engages within the first tick and the evaluation window $[8,30]\,\si{\second}$ excludes the transient.

\subsection{Reduction-Fidelity Check}
\label{sec:VI-reduction}

Two distinct quantities must be kept apart: the \emph{per-rope shape-mode deviation} $|T_i^{\mathrm{KV}}-T_i^{\mathrm{qs}}|$ that \cref{thm:reduction} bounds at $\le C_1\delta+C_2\eta_{\max}\approx\SI{0.076}{\newton}$, and the \emph{inter-rope asymmetry} $\varepsilon_i\triangleq T_i^{\mathrm{KV}}-T_{\mathrm{eff}}$ around the surviving-rope mean, which is formation-geometry driven and not bounded by the theorem. Across $[8,30]\,\si{\second}$ on V1--V6, $|\varepsilon_i|$ has RMS $7$--$13$\,N (P95 up to $24$\,N, peaks $\sim\SI{60}{\newton}$ at the most centripetally-loaded instants); after a fault each survivor carries more load and formation angles differ, so $T_i^{\mathrm{qs}}$ varies across ropes by $O(10)$\,N. Because $T_i^{\mathrm{ff}}=T_i$ uses the measured per-rope tension, each drone receives the local asymmetric load signal; the stability analysis therefore treats this asymmetry through the measured input rather than by an explicit inter-rope model (\cref{lem:tension-cancel}). The domain-gate check \cref{tab:domain-audit} is the primary check.

% ----------------------------------------------------------------
\subsection{Nominal Tracking Performance}
\label{sec:VI-B}
% ----------------------------------------------------------------

V1 (no wind) yields 3-D RMSE $\SI{0.312}{\metre}$ over $[8,30]$~s. The bound of \cref{cor:steady-state} is horizontal-only; V1 horizontal RMSE $\SI{0.306}{\metre}$ matches the predicted $\SI{0.27}{\metre}$ to leading order, with the gap from centripetal acceleration on the curved segments. The 3-D value exceeds the corollary because it accumulates altitude error over the $\pm\SI{0.35}{\metre}$ vertical excursion, which the horizontal-only derivation excludes.

V2 (Dryden wind, $\SI{4}{\metre\per\second}$ mean, seed~42) yields RMSE $\SI{0.312}{\metre}$, indistinguishable from V1 at single-seed precision. The PD/QP layer absorbs the turbulence without adaptive augmentation, exercising the wind margin of \cref{cor:robustness}. \Cref{tab:performance-ci} consolidates per-variant RMSE, sag, and tension; \cref{fig:trajectory-3d-V4} pairs the V4 payload trajectory in 3-D with a representative Drake simulation view of the same dual-fault scenario.

\Cref{tab:performance-ci} reports integrated system-level performance for V1--V6 (deterministic seed-42 estimates). The rightmost column marks each variant against the pre-registered thresholds: 3-D RMSE $\le \SI{0.35}{\metre}$, sag $\le \SI{100}{\milli\metre}$, tension $\le \SI{120}{\newton}$.

\begin{table}[htbp]
  \caption{Integrated system-level performance (V1--V6, seed-42). Acceptance thresholds: RMSE $\le \SI{0.35}{\metre}$; sag $\le \SI{100}{\milli\metre}$; tension $\le \SI{120}{\newton}$.}
  \label{tab:performance-ci}
  \centering
  \small
  \renewcommand{\arraystretch}{1.05}
  \begin{tabular}{lcccc}
    \toprule
    Variant & RMSE & Sag\textsuperscript{peak}
            & Tension\textsuperscript{peak} & Acceptance \\
    \midrule
    V1 & 0.312m & ---  & 105.3N & Pass \\
    V2 & 0.312m & ---  & 104.6N & Pass \\
    V3 & 0.324m & 86.8mm & 104.6N & Pass \\
    V4 & 0.328m & 95.2mm & 104.6N & Pass \\
    V5 & 0.324m & 90.4mm & 104.6N & Pass \\
    V6 & 0.326m & 76.1mm & 107.6N & Pass \\
    \bottomrule
  \end{tabular}
\end{table}

% ----------------------------------------------------------------
\subsection{Fault Progression, Recovery, and Lyapunov Evidence}
\label{sec:VI-C}
% ----------------------------------------------------------------

We evaluate \cref{thm:hybrid-stability} against three scenarios of increasing severity: V3 (single fault), V4 (dual, $\Delta t = \SI{5}{\second}$), V5 (dual, $\Delta t = \SI{10}{\second}$).

\paragraph{V3 --- single fault.} Drone~0 severs at $t = \SI{12}{\second}$; four survivors respond without fault signals or peer communication. Peak 3-D error $\SI{0.337}{\metre}$, IAE $\SI{0.452}{\metre\second}$, max sag $\SI{85.1}{\milli\metre}$. Settle time is below the \SI{1}{\milli\second} tick resolution (logged as 0~s): the bounded jump of \cref{lem:jump-continuity} at the near-equal-sharing hover point falls below one tick.

\paragraph{V4 --- dual fault, $\Delta t = \SI{5}{\second}$.} Drone~2 severs at $t = \SI{17}{\second}$, $\SI{5}{\second}$ after the first. Peak error $\SI{0.446}{\metre}$, IAE $\SI{0.666}{\metre\second}$, max sag $\SI{79.1}{\milli\metre}$, settle time $\SI{0.158}{\second}$. The second-fault transient decays within the same envelope, consistent with $\rho < 1$.

\paragraph{V5 --- dual fault, $\Delta t = \SI{10}{\second}$.} Extending the gap to $\SI{10}{\second}$ lets the system fully re-settle. The second fault's IAE is $\SI{0.741}{\metre\second}$ ($+11\%$ over V4) reflecting higher cruise speed at $t = \SI{22}{\second}$, not degraded recovery. Max sag $\SI{87.4}{\milli\metre}$; sub-tick settle.

\paragraph{V6 --- full stack.} With MPC~$+L_1$+reshape active, first-fault sag is $\SI{69.9}{\milli\metre}$ ($-12\%$ vs V4), consistent with the stack preserving the envelope while delivering incremental benefit.

\Cref{fig:fault-zoom} shows the fault-centered timeline for V4 with FF-on (solid) vs FF-off (dashed) on the same wind seed: (a)~tracking-error norm, (b)~altitude error exposing the sag differential, (c)~Lyapunov proxy $V(t) = \xi^\top P_v \xi$, allowing direct visual inspection of the bounded jump (\cref{lem:jump-continuity}) and the inter-event contraction (\cref{lem:dwell-decay}) together with mechanism evidence from the FF ablation. V3, V5, FF-held-constant, and per-drone thrust traces are in the supplementary package. Empirical fits on V4 and V5 give peak jump $\hat\chi_{\max} = 1.1\!\times\!10^{-4}$ (V4), $1.3\!\times\!10^{-5}$ (V5)---in units of the dimensionless proxy $V$, not physical units---and dwell-cycle contraction $\hat\rho_{\text{V4}}=0.20$, $\hat\rho_{\text{V5}}=0.045$, both well below the $\rho < 1$ sufficiency bound (\cref{tab:stability-evidence}).

\begin{figure}
  \centering
  \subfloat[Tracking-error norm (V4).]{%
    \includegraphics[width=0.95\columnwidth]{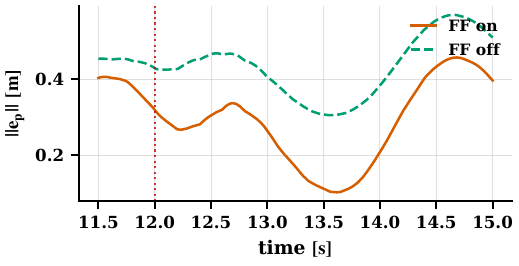}%
    \label{fig:fz-trackerr-v4}}\hfill
  \subfloat[Payload altitude error (V4).]{%
    \includegraphics[width=0.95\columnwidth]{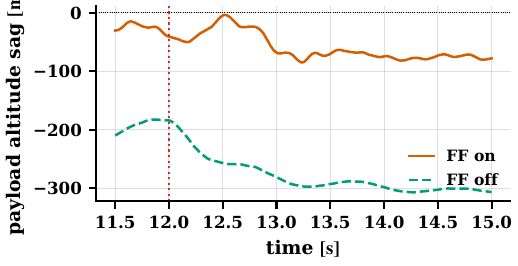}%
    \label{fig:fz-altitude-v4}}\hfill
  \subfloat[Lyapunov proxy $V=\xi^\top P_v\xi$ (V4).]{%
    \includegraphics[width=0.95\columnwidth]{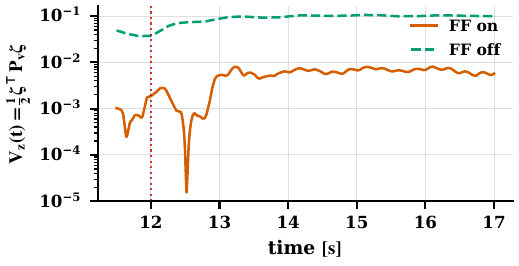}%
    \label{fig:fz-lyapunov-v4}}
  \caption{Fault-centered causal ablation on V4, FF-on (solid) vs FF-off (dashed, P2-A) at seed~42. \textit{(a)}~tracking-error norm; \textit{(b)}~altitude error; \textit{(c)}~Lyapunov proxy $V(t)=\xi^\top P_v\xi$, exhibiting the bounded jump of Lemma~\ref{lem:jump-continuity} and the inter-event contraction of Lemma~\ref{lem:dwell-decay}.}
  \label{fig:fault-zoom}
\end{figure}

\subsection{Tension Feedforward Ablation}
\label{sec:VI-D}

The P2-A campaign provides the strongest quantitative evidence for \eqref{eq:Tff} as the causal recovery mechanism. We disable $T_i^{\mathrm{ff}} = T_i$ in the QP layer and repeat V3, V4, V5.

Without the identity, the altitude channel carries an unrejected load increment, inflating both steady-state floor and post-fault Lyapunov jumps (\cref{lem:jump-continuity}); \cref{tab:ff-ablation,fig:fault-zoom} summarize the $34$--$39\%$ RMSE increase and $3.6$--$4.0\times$ peak-sag inflation.

\begin{table}[t]
  \caption{Tension feed-forward ablation (P2-A). Cruise-window RMSE (m) and peak post-fault altitude sag (mm) with FF enabled and disabled, per fault schedule.}
  \label{tab:ff-ablation}
  \centering
  \small
  \renewcommand{\arraystretch}{1.05}
  \begin{tabular}{lcccccc}
    \toprule
    & \multicolumn{3}{c}{RMSE (m)} & \multicolumn{3}{c}{Peak sag (mm)} \\
    \cmidrule(lr){2-4}\cmidrule(lr){5-7}
    Variant & FF on & FF off & $\Delta$\% & FF on & FF off & $\times$ \\
    \midrule
    V3 & 0.324 & 0.432 & $+34\,\%$ & 86.8 & 309 & $3.6\times$ \\
    V4 & 0.328 & 0.457 & $+39\,\%$ & 95.2 & 358 & $3.8\times$ \\
    V5 & 0.324 & 0.444 & $+37\,\%$ & 90.4 & 359 & $4.0\times$ \\
    \bottomrule
  \end{tabular}
\end{table}

\begin{figure}[htbp]
  \centering
  \subfloat[Surviving rope tension $T_1(t)$.]{%
    \includegraphics[width=0.96\columnwidth]{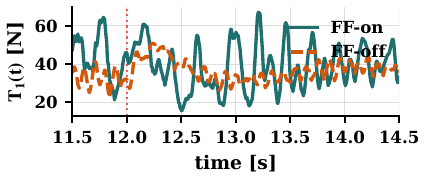}%
    \label{fig:self-announcement-a}}\\[-0.6ex]
  \subfloat[Feed-forward command $T_1^{\mathrm{ff}}(t)$.]{%
    \includegraphics[width=0.96\columnwidth]{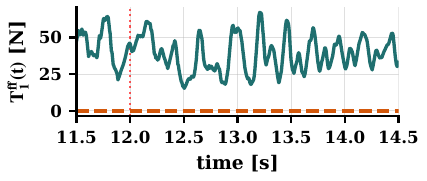}%
    \label{fig:self-announcement-b}}\\[-0.6ex]
  \subfloat[Commanded thrust $f_1(t)$.]{%
    \includegraphics[width=0.96\columnwidth]{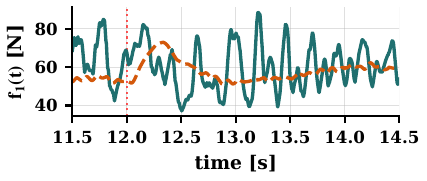}%
    \label{fig:self-announcement-c}}\\[-0.6ex]
  \subfloat[Payload altitude error $e_{L,z}(t)$.]{%
    \includegraphics[width=0.96\columnwidth]{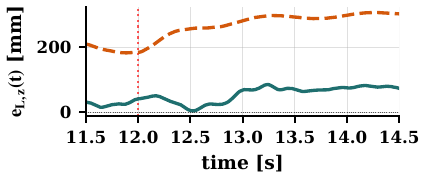}%
    \label{fig:self-announcement-d}}
  \caption{Self-announcement causal chain for V4 fault~1, where drone~0 is severed at $t=\SI{12}{\second}$. Teal solid denotes the baseline FF-on case, and orange dashed denotes the P2-A FF-off ablation. With the identity active, $T_1^{\mathrm{ff}}$ tracks $T_1$ at the next tick, thrust steps proportionally, and altitude error remains bounded. With the identity disabled, $T_1^{\mathrm{ff}}\!=\!0$ remains flat, the thrust step is absent, and altitude error grows by approximately \SI{300}{\milli\metre}.}
  \label{fig:self-announcement}
\end{figure}

% ----------------------------------------------------------------
\subsection{Dwell-Time Boundary Characterisation}
\label{sec:VI-dwell}
% ----------------------------------------------------------------

\eqref{eq:dwell-assumption} asserts $\rho = e^{-\alpha_{\min}\tau_{\mathrm{pend}}} < 1$ for any $\tau_d > \tau_{\mathrm{pend}}$. V4 and V5 confirm two operating points inside the guaranteed regime but do not characterize behavior as $\tau_d$ approaches or falls below $\tau_{\mathrm{pend}}$. We close this gap with a boundary sweep on V4 across $\tau_d/\tau_{\mathrm{pend}} \in \{0.5, 0.75, 1.0, 1.25, 1.5, 2.0\}$ at seed~42.

For each point we compute the per-fault peak altitude excursion $\hat e_{p,z}^{\max}(t_k)$ relative to a 0.2-second DC-offset window ending at $t_k$ (removing the post-fault equilibrium shift), then form $\hat\rho(\tau_d) = \hat e_{p,z}^{\max}(t_2)/\hat e_{p,z}^{\max}(t_1)$. \Cref{fig:dwell-sweep} reports the result.

The empirical result is stronger than the hypothesis demands: $\hat\rho < 1$ at every tested dwell, including the sub-threshold $\tau_d=0.5\tau_{\mathrm{pend}}$ where $\hat\rho\approx 0.63$. The ratio falls to $\hat\rho\approx 0.14$ at $\tau_d=1.5\tau_{\mathrm{pend}}$. The non-monotonicity at $\tau_d=2.0\tau_{\mathrm{pend}}$ reflects the trajectory phase the second fault hits (centripetal peak vs trough); a fault-phase sweep would resolve it. The fault-2 peak is consistently below fault-1 across the sweep: the identity bounds the Lyapunov jump at each fault event (\cref{lem:jump-continuity}), and the inter-event contraction at rate $\alpha_{xy}$ (\cref{lem:horizontal-iss})---inducing the norm-level recovery rate $\lambda = \alpha_{xy}/2 = \SI{1.19}{\second^{-1}}$ of \cref{thm:hybrid-stability}---is sufficient to dissipate that bounded jump within each dwell interval, including the sub-threshold cases. H2 remains a structurally motivated bound; the empirical map shows contraction margin persists below it on this trajectory family.

\subsection{Parameter Mismatch and Adaptive Augmentation}
\label{sec:VI-E}

P2-B evaluates robustness to payload-mass mismatch and the $L_1$ benefit (\cref{sec:l1-augmentation}, C3-$L_1$). We sweep $m_L \in \{2.5, 3.0, 3.5, 3.9\}~\si{\kilo\gram}$ (the PD design range, distinct from the $\SI{10}{\kilo\gram}$ V1--V6 payload) with $q^\star$ fixed at $\SI{3.0}{\kilo\gram}$ and compare FF-on, FF-off, FF-off$+L_1$. All runs use the fault-free trajectory to isolate mismatch effects. The $L_1$ evidence is obtained in rescue mode (FF off, $L_1$ alone), which isolates the $L_1$ contribution; main-mode evaluation at the canonical \SI{10}{\kilo\gram} point is follow-up work.

FF-on is robust across the range: RMSE varies $<1.9\%$, confirming the identity pre-compensates the per-drone load change without explicit mass estimation (\cref{cor:robustness}(i)). FF-off degrades approximately linearly, reaching $\SI{0.350}{\metre}$ RMSE at $m_L = \SI{3.9}{\kilo\gram}$ vs $\SI{0.341}{\metre}$ at the nominal. FF-off$+L_1$ recovers $\sim 40\%$ of the mismatch-induced sag.

The single-axis sweep of \cref{fig:l1-mass-mismatch} plots the altitude-sag channel, where the mismatch effect and the $L_1$ recovery are most visible. It is one projection of the four-axis claim of \cref{cor:robustness}. The full response surface over $(m_L, k_s, c_{\mathrm{drag}}, \bar{w})$ on $n=100$ CRN-paired seeds is identified as a campaign extension.

\paragraph{$L_1$ adaptive-gain stability map.} A sweep of $\Gamma\in\{0.5,2,10,30,50,80\}\times 10^3$ in rescue mode ($m_L=\SI{3.9}{\kilo\gram}$, $q^\star=\SI{3.0}{\kilo\gram}$, FF off, no faults) confirms the admissible window of \cref{prop:gamma-star}. Altitude RMSE peaks at $\Gamma=500<\Gamma_{\min}$ ($\SI{101}{\milli\metre}$), reaches a minimum of $\SI{88.5}{\milli\metre}$ at $\Gamma\approx 10^4$ inside the window, and rises mildly to $\SI{91.4}{\milli\metre}$ at $\Gamma=8\times 10^4$. The unimodal shape with worst performance below $\Gamma_{\min}$ empirically supports both window bounds; no instability is observed across the swept range.

% ----------------------------------------------------------------
\subsection{Extension Characterization and Binding-Regime
  Evidence}
\label{sec:VI-H}
\label{sec:VI-F}
% ----------------------------------------------------------------

P2-C and P2-D probe the MPC tension-ceiling (C3-MPC) and reshape supervisor (C3-Reshape) on V4. Both yield null findings on the canonical family, complemented by binding-regime probes for the operating conditions under which each extension activates.

\paragraph{P2-C --- MPC tension ceiling.} We sweep $T_{\max} \in \{60, 70, 80, 90, 100\}~\si{\newton}$ against the unconstrained baseline (\cref{tab:mpc-ceiling-sweep}).

\begin{remark}[NF2: MPC ceiling non-binding in canonical regime]
\label{rem:mpc-null}
  Peak tension during the fault transient is dominated by the pre-fault dynamic spike from the lemniscate's centripetal acceleration ($\approx \SI{42}{\newton}$ on survivors immediately before the first fault, resolving within one pendulum period); this spike is independent of the ceiling. The ceiling targets the post-fault quasi-static regime: at the 5$\to$3 transition with $\SI{10}{\kilo\gram}$ payload, $m_L g / 3 \approx \SI{32.7}{\newton}$, well below all tested ceilings. Lowering the ceiling from $\SI{100}{\newton}$ to $\SI{60}{\newton}$ produces no measurable change (NF2). The ceiling becomes binding when the post-fault load approaches $T_{\max}$---heavier payload, lower ceiling, or fewer survivors. \Cref{prop:mpc-slack,sec:theorem-preservation} preserve \cref{thm:hybrid-stability} whether or not the ceiling is binding.
\end{remark}

\paragraph{P2-D --- Formation reshape.} We sweep $T_{\mathrm{ref}} \in \{8, 10, 12\}~\si{\second}$ with and without reshape on V4, plus the aggressive-period binding-regime probe at $T_{\mathrm{ref}}=\SI{6}{\second}$ (E8). Peak post-fault rope tension is identical with reshape on and off at every tested period: $\SI{330.2}{\newton}$, $\SI{191.4}{\newton}$, $\SI{153.8}{\newton}$, $\SI{87.9}{\newton}$ at $T_{\mathrm{ref}} = 6, 8, 10, 12~\si{\second}$ respectively.

\begin{remark}[NF3: Reshape benefit zero across periods and aggressive-period probe]
\label{rem:reshape-null}
  The reshape supervisor targets the quasi-static hover tension distribution under the symmetric load-sharing model of \cref{thm:reshape}. At the tested periods, centripetal acceleration $a_c = (2\pi/T_{\mathrm{ref}})^2 A \approx 0.82$--$3.29~\si{\metre\per\second\squared}$ continuously modulates the per-rope distribution, and this dynamic modulation dominates the hover-equilibrium asymmetry the supervisor targets---so high centripetal load, far from making the asymmetry binding, makes it \emph{non}-binding. The quasi-static assumption of \cref{thm:reshape} therefore fails to bind at any tested period, including the aggressive $T_{\mathrm{ref}}=\SI{6}{\second}$ probe (E8) for which the linearised theorem a priori predicted a $\sim 25.8\%$ reduction; the measured benefit there is $0.0\%$. Pushing reshape into a binding regime would require either a lower-bandwidth reference family or an extended objective that accounts for the dynamic load. \Cref{thm:reshape,sec:theorem-preservation} preserve the envelope outside the binding regime.
\end{remark}

\begin{table}[t]
  \caption{MPC tension-ceiling sweep (P2-C, NF2). Peak rope tension and fraction of mission time over the commanded ceiling $T_{\max}$, V4 fault schedule.}
  \label{tab:mpc-ceiling-sweep}
  \centering
  \small
  \renewcommand{\arraystretch}{1.05}
  \begin{tabular}{crrr}
    \toprule
    Ceiling $T_{\max}$ & Peak $T$ & Peak post-fault & Time over ceiling \\
    \midrule
    60N       & 104.6N & 87.9N & 17.52\% \\
    70N       & 104.6N & 87.9N & 6.50\%  \\
    80N       & 104.6N & 87.9N & 1.60\%  \\
    90N       & 104.6N & 87.9N & 0.34\%  \\
    100N      & 104.6N & 87.9N & 0.12\%  \\
    Baseline & 104.6N & 87.9N & ---   \\
    \bottomrule
  \end{tabular}
\end{table}

A trajectory-phase sweep on V4 ($t_1\in\{8,10,12,14,16\}\,\si{\second}$) shows peak post-fault sag varying $\SI{12}{\milli\metre}$--$\SI{55}{\milli\metre}$ and 3-D cruise RMSE within $\sim 7\%$; the canonical $t_1=\SI{12}{\second}$ sits at the worst-case end, strengthening rather than weakening the recovery claim. A sub-threshold dwell probe at $\tau_d=0.5\,\tau_{\mathrm{pend}}$ on the same V4 schedule remains bounded with $\hat\rho<1$ (\cref{fig:dwell-sweep}), showing contraction margin beyond H2 sufficiency on this family.

\subsection{Actuator-Margin and Recovery-Time Check}
\label{sec:VI-actuator-recovery}

The concern with cascade architectures is that the controller may rely on actuator saturation to absorb post-fault loads. We evaluate this using the per-drone maximum commanded thrust ratio $f_i(t)/f_{\max}$ over $[8,30]\,\si{\second}$ for V3--V6. The variant maxima are $0.73$ on V3, V4, V5 and $0.79$ on V6, all attained on drone~4, with zero cruise-window time above $0.9\,f_{\max}$ on every variant. The V3--V5 maxima coincide because the three variants share the first-fault schedule (drone 0 severs at $t=\SI{12}{\second}$) and the simulations are deterministic on a common wind seed: the variant peak is drone 4's transient at $t=\SI{12.68}{\second}$, $\SI{0.68}{\second}$ into the fault-1 recovery, and the second-fault spike (V4: $t=\SI{17}{\second}$, V5: $t=\SI{22}{\second}$) lands on a different drone with the trajectory at a different phase and stays below the fault-1 peak on drone 4. V6 reaches $0.79$ because the full-stack extensions ($L_1$ + MPC + reshape) shift the altitude command. Recovery is not actuator-limited on any variant.

\paragraph{Spike-then-relax dynamics.} On V4, the three survivors show a $1.18$--$1.79\times$ peak-over-settled overshoot decaying within $\sim\tau_{\mathrm{pend}}$, with steady-state thrust per survivor rising $+25\%$ to $+34\%$ over the pre-fault baseline. The survivor sum at $t\in[27,30]\,\si{\second}$ settles near $\SI{179}{\newton}$, balancing $m_L g+3m_{\mathrm{drone}}g\approx\SI{142}{\newton}$ plus centripetal/tilt overhead---an independent actuator-channel check of post-fault load redistribution.

\Cref{tab:recovery-iae} reports per-fault recovery metrics beyond cruise RMSE: peak error norm, peak sag, recovery time $t_{\mathrm{rec}}$ to within $\varepsilon=\SI{0.35}{\metre}$ for $\SI{0.3}{\second}$ continuous, and IAE over $\tau_{\mathrm{pend}}$ post-fault.

\begin{table}[t]
  \caption{Per-fault recovery metrics. $t_{\mathrm{rec}}$ threshold $\SI{0.35}{\metre}$ (3-D, $\SI{0.3}{\second}$ continuous hold); IAE integrated over $\tau_{\mathrm{pend}}$ or the gap to the next fault, whichever is shorter.}
  \label{tab:recovery-iae}
  \centering
  \small
  \begin{tabular}{llccccc}
    \toprule
    Variant & Fault & $t^\star$ & peak $\|e\|$ &
      peak sag & $t_{\mathrm{rec}}$ & IAE \\
    \midrule
    V3 & 1 & 12.0s & 481mm & 85mm & 0.00s & 0.52m$\cdot$s \\
    V4 & 1 & 12.0s & 481mm & 85mm & 0.00s & 0.52m$\cdot$s \\
    V4 & 2 & 17.0s & 457mm & 89mm & 0.32s & 0.70m$\cdot$s \\
    V5 & 1 & 12.0s & 481mm & 85mm & 0.00s & 0.52m$\cdot$s \\
    V5 & 2 & 22.0s & 469mm & 90mm & 1.58s & 0.81m$\cdot$s \\
    V6 & 1 & 12.0s & 481mm & 73mm & 0.00s & 0.51m$\cdot$s \\
    V6 & 2 & 17.0s & 460mm & 72mm & 0.33s & 0.70m$\cdot$s \\
    \bottomrule
  \end{tabular}
\end{table}

%% file: Section_VII_Discussion.tex
The campaign of \cref{sec:simulation} establishes a single architectural commitment: the announcement-detection-reconfiguration timeline is unnecessary when the structural transition appears locally through actuator-side measurements, and the identity $T_i^{\mathrm{ff}}=T_i$ suffices. The five-drone, ten-kilogram dual-fault scenario is one point on the feasibility surface defined by \eqref{eq:actuator-envelope}.

\subsection{Check of the Two Main Contributions}
\label{sec:discussion-audit}

\Cref{thm:reduction} (C1) is supported in two layers. The domain-gate check (\cref{tab:domain-audit}) shows that all six missions sit inside $\Omega^{\mathrm{dwell}}_{\tau}$ with non-trivial margin. The fidelity layer (\cref{sec:VI-reduction}) separates the shape-mode deviation $|T_i^{\mathrm{KV}} - T_i^{\mathrm{qs}}|$ (bounded by \cref{thm:reduction} at $\approx\SI{0.076}{\newton}$ per rope) from the inter-rope tension asymmetry $|\varepsilon_i|$ (RMS $\SI{12.4}{\newton}$ on dual-fault variants, formation-geometry driven, not bounded by the theorem). The identity $T_i^{\mathrm{ff}} = T_i$ absorbs the asymmetry per-drone without centralized averaging.

\Cref{thm:hybrid-stability} (C2) is supported by three converging measurements: all six variants pass the pre-registered thresholds with dual-fault peak sag $\SI{95.2}{\milli\metre}$ vs $\SI{100}{\milli\metre}$ ceiling (\cref{tab:performance-ci}); fits on V4/V5 Lyapunov proxy give $\hat\rho_{\mathrm{V4}}=0.20$, $\hat\rho_{\mathrm{V5}}=0.045$; the dwell sweep shows $\hat\rho < 1$ across $\tau_d/\tau_{\mathrm{pend}}\in[0.5,3.0]$ including the sub-threshold point.

The simulation audit further characterizes extension layers ($L_1$, MPC, reshape) through targeted sweeps. The $L_1$ adaptive-gain sweep shows altitude RMSE is highest at $\Gamma=500 < \Gamma_{\min}$ (bandwidth-coverage violation) and drops smoothly into the window, reaching a minimum at $\Gamma \approx 10^4$ before mild recovery at $\Gamma=8\times10^4$---the unimodal shape with worst performance at the below-$\Gamma_{\min}$ point is the empirical signature of the two-sided admissible window \eqref{eq:gamma-window}. Both main contributions converge on a single causal anchor: the feed-forward ablation (P2-A) raises cruise RMSE by $34$--$39\%$ and inflates sag $3.6$--$4.0\times$ (\cref{tab:ff-ablation,fig:self-announcement}), consistent across V3, V4, V5.

\subsection{What the Architecture Changes}
\label{sec:discussion-advance}

Existing decentralized cooperative-transport architectures address fault tolerance through one of three families. Active fault-tolerant pipelines~\cite{LiuSuspensionFailure2021} estimate, classify, and reassign on a detection-confirmation-reassignment timeline comparable to or longer than $\tau_{\mathrm{pend}}$. Passive force-control and leader--follower architectures~\cite{GassnerPassiveForce2017,NoCommTransport2021} avoid inter-UAV exchange but presuppose fixed cable cardinality, with no guarantee across the constraint-set transition. Centralized and synchronous-distributed force allocation~\cite{Michael2011,WangOng2017,LiLoianno2023,PallarLi2025,SuBhowmick2023} is structurally inapplicable: the per-tick global-state synchronization it requires is precisely what the constraint-cardinality transition breaks at the actuator-tick rate.

The controller maps the transition to thrust through $T_i^{\mathrm{ff}}=T_i$, eliminating the detection-reconfiguration timeline rather than substituting a faster version. The corresponding measurements are $\SI{0.00}{\second}$ recovery on first faults (sub-tick) and $\SI{0.32}{\second}$, $\SI{0.33}{\second}$ on V4/V6 second faults, well inside one pendulum period (\cref{tab:recovery-iae}). This implementation has four practical properties: \emph{locality} by $\mathcal{I}_i$ \eqref{eq:info-set}; \emph{passivity} by the absence of detector, classifier, or scheduler in the loop; \emph{robustness to multi-agent failure} by induction on the fault counter, with the demonstrated $(N{=}5,m_L{=}\SI{10}{\kilo\gram},F{=}2)$ consuming $27\%$ of the actuator envelope \eqref{eq:actuator-envelope}; and \emph{team-size independence}, since the per-drone law and lemma stack carry no $N$ structurally and the dependence concentrates in $m_Lg/N$ and $m_Lg/(N{-}F)$.

The two null findings (NF2, NF3) help delimit the regimes in which the baseline term is dominant. MPC shows identical peak tension $\approx \SI{104.6}{\newton}$ across all ceilings because the pre-fault dynamic spike dominates the targeted quasi-static load (\cref{tab:mpc-ceiling-sweep}). The reshape supervisor shows $0.0\%$ benefit because centripetal load dominates the hover asymmetry \cref{thm:reshape} targets (\cref{rem:reshape-null}). Together, these results indicate that the additional active-FT layers are not the limiting factors on the demonstrated family; \cref{sec:theorem-preservation} states the corresponding preservation conditions for \cref{thm:hybrid-stability}.

\subsection{Boundary of the Analysis Domain}
\label{sec:discussion-boundary}

Three regimes lie outside the analysis domain: rapid multi-fault ($\tau_d < \tau_{\mathrm{pend}}$) voids \cref{lem:dwell-decay}; long slack excursions exit $\Omega^{\mathrm{dwell}}_{\tau}$; QP fallback breaks D1. A multi-fault extension of the dwell argument would further clarify these operating boundaries.

A second class of future work concerns linearization residuals: tube-MPC folding the recovery-envelope overshoot into the QP RHS would close P2-C; extended-objective reshape would close P2-D; main-mode $L_1$ at $\SI{10}{\kilo\gram}$ would complete C3-$L_1$. A third class is not addressed by the present model---drone-side faults, soft cable failures, tension-sensor failures, and payload-attitude dynamics--- as enumerated in \cref{sec:theorem-preservation}.

Within this domain, the campaign indicates that the identity $T_i^{\mathrm{ff}}=T_i$ delivers hybrid practical input-to-state stability with measured contraction margin across single and dual unannounced severances, that the identity is necessary in the precise sense P2-A isolates, and that the additional active-FT layers do not materially affect the trajectory family for which the feed-forward term is dominant.

%% file: Section_VIII_Conclusion.tex
The results show that a decentralized tension feed-forward term can improve recovery after abrupt cable severance in cooperative aerial transport. By setting each drone's altitude thrust feed-forward equal to its measured cable tension, the controller maps load redistribution into a local thrust adjustment at the next control update. The method uses local measurements and does not require a separate fault flag, inter-agent communication, or online controller reconfiguration for the severance cases considered here.

We analyzed the resulting closed loop as a hybrid system and established a practical input-to-state-stability result under structural fault transitions. The analysis provides an explicit recovery envelope with exponential decay and per-fault-cycle contraction, and it gives bounded tracking behavior under the stated actuator, slack-excursion, and dwell-time assumptions. A reduced-order cable model supports the analysis by bounding the effect of distributed cable dynamics on the closed loop.

High-fidelity multibody simulations show recovery from single- and multiple-cable severance events within one payload-pendulum period on the reported mission family. Feed-forward ablations indicate that the tension feed-forward term is the main contributor to this behavior in the tested baseline cascade. The evaluated extension layers do not materially alter the trajectories in the regimes studied, which helps identify the operating range in which the baseline mechanism is dominant.

These results support the use of local force measurements as response channels for abrupt structural transitions in cooperative aerial transport. Future work will extend the analysis beyond the slack-excursion-bounded domain, investigate hardware validation, and study related multi-agent and contact-rich systems with additional sensing and fault modes.